\documentclass[conference,letterpaper]{IEEEtran}

\usepackage[utf8]{inputenc} 
\usepackage[T1]{fontenc}
\usepackage{url}
\usepackage{subcaption}
\usepackage{ifthen}
\usepackage{graphicx, color}
\usepackage{algorithm}
\usepackage{algpseudocode}
\usepackage{wrapfig}
\usepackage{comment}
\usepackage{cite}
\usepackage[cmex10]{amsmath} 
\usepackage{amsfonts}       
\usepackage{amssymb}
\usepackage[english]{babel}
\usepackage{amsthm}
\newtheorem{theorem}{Theorem}[section]
\newtheorem{corollary}{Corollary}[theorem]
\newtheorem{lemma}[theorem]{Lemma}
\newtheorem{proposition}[theorem]{Prop.}
\theoremstyle{definition}
\newtheorem{definition}{Definition}[section]

\begin{document}
\title{ \vspace{-15pt}\huge Mixed Dimension Embeddings with Application to Memory-Efficient Recommendation Systems  \vspace{-10pt}
} 



 \author{%
   \IEEEauthorblockN{A.A. Ginart\IEEEauthorrefmark{1},
                     Maxim Naumov\IEEEauthorrefmark{2},
                     Dheevatsa Mudigere\IEEEauthorrefmark{2},
                     Jiyan Yang\IEEEauthorrefmark{2},
                     James Zou\IEEEauthorrefmark{1}}
   \IEEEauthorblockA{\IEEEauthorrefmark{1}%
                     \small Stanford University,
                     Palo Alto, California,
                     \{tginart, jamesz\}@stanford.edu}
   \IEEEauthorblockA{\IEEEauthorrefmark{2}%
                     \small Facebook, Inc.
                     Menlo Park, California,
                      \{mnaumov, dheevatsa, chocjy\}@facebook.com}
 \vspace{-30pt}}

\maketitle

\begin{abstract}
 \baselineskip=9pt \footnotesize Embedding representations power machine intelligence in many applications, including recommendation systems, but they are space intensive --- potentially occupying  hundreds of gigabytes in large-scale settings. To help manage this outsized memory consumption, we explore \emph{mixed dimension embeddings}, an embedding layer architecture in which a particular embedding vector's dimension scales with its query frequency. Through theoretical analysis and systematic experiments, we demonstrate that using mixed dimensions can drastically reduce the memory usage, while maintaining and even improving the ML performance. Empirically, we show that the proposed mixed dimension layers improve accuracy by 0.1\% using half as many parameters or maintain it using 16$\times$ fewer parameters for click-through rate prediction on the Criteo Kaggle dataset. They also train over 2$\times$ faster on a GPU.
\end{abstract}



\normalsize
\vspace{-7pt}
\section{Introduction}

Embedding representations power state-of-the-art applications in diverse domains, including computer vision \cite{barz2019hierarchy,vasileva2018learning}, natural language processing \cite{shoeybi2019megatron, akbik2018contextual, liu2019roberta}, and recommendation systems \cite{cheng2016wide, Park2018, Wu2019}. It is standard practice to embed objects into $\mathbb{R}^d$ at a fixed uniform dimension (UD) $d$. When the embedding dimension $d$ is too low, the downstream statistical performance suffers  \cite{yin2018dimensionality}. When $d$ is high and the number of objects to represent is large, memory consumption becomes an issue. For example, in recommendation models, the embedding layer can make up more than $99.9\%$ of the memory it takes to store the model, and in large-scale settings, it could consume hundreds of gigabytes or even terabytes \cite{Park2018, Pi2019}. Therefore, finding innovative embedding representations that use fewer parameters while preserving statistical performance of the downstream model is an important challenge.  

Object frequencies are often heavily skewed in real-world applications. For instance, for the full MovieLens dataset, the top 10\% of users receive as many queries as the remaining 90\% and the top 1\% of items receive as many queries as the remaining 99\%. To an even greater extent, on the Criteo Kaggle dataset the top $0.0003\%$ of indices receive as many queries as the remaining $\sim$32 million. To leverage the heterogeneous object popularity in recommendation, we propose mixed dimension (MD) embedding layers, in which the dimension of a particular  object's embedding scales with that object's popularity rather than remaining uniform. Our case studies and  theoretical analysis demonstrate that MD embeddings work well because they do not underfit popular embeddings while wasting parameters on rare embeddings. Additionally, MD embeddings minimize popularity-weighted loss at test time by efficiently allocating parameters.



In Section 3, we introduce the proposed architecture for the embedding layer. In Section 4, we theoretically investigate MD embeddings. Our theoretical framework splits embedding-based recommendation systems into either the \emph{data-limited regime} or the \emph{memory-limited regime}, depending on the parameter budget and sample size. We prove mathematical guarantees, which demonstrate that when the frequency of categorical values is sufficiently skewed, MD embeddings are both better at matrix recovery and incur lower reconstruction distortion than UD embeddings. Our method is faster to train while requiring far less tuning than other non-uniform embedding layers. In Section 5, we demonstrate that MD embeddings improve both parameter-efficiency and training time in click through rate (CTR) prediction tasks. 

\textbf{Summary of Contributions:}

    \textbf{(1)} We propose an MD embeddings layer for recommendation systems and provide a novel, mathematical method for sizing the dimension of features with variable popularity that is fast to train, easy to tune, and performs well empirically.
    
   \textbf{(2)}   With matrix completion and factorization models, we prove that with sufficient popularity skew, mixed dimension embeddings incur lower distortion when memory-limited and generalize better when data-limited.
    
    \textbf{(3)}    For the memory-limited regime we derive the \emph{optimal} feature dimension. This dimension only depends on the feature's \emph{popularity}, the parameter budget, and the \emph{singular-value spectrum} of the pairwise interactions. 

\section{Background \& Problem Formulation}


We review the CTR prediction task here (more details in Appendix). Compared to canonical collaborative filtering (CF),  CTR prediction tasks include additional context that can be incorporated to predict user-item interactions.  These contextual features are expressed through sets of indices (categorical) and floating point values (continuous). These features can represent arbitrary details about the context of an on-click or personalization event. The $i$-th categorical feature can be represented by an index $x_i \in \{1,...,n_i\}$ for $i=1,...,\kappa$. In addition to $\kappa$ categorical features, we also have $s$ scalar features, together producing a dense feature vector $\textbf{x}' \in \mathbb{R}^s$. Thus, given some $(x_1,...,x_{\kappa},\textbf{x}') \in ([n_1] \times ... \times [n_{\kappa}]) \times \mathbb{R}^s$, we would like to predict $y \in \{0,1\}$, which denotes a click event in response to a particular personalized context. 

We  use state-of-the-art deep learning recommendation model (DLRM) \cite{naumov2019deep} as an off-the-shelf deep model. Various deep CTR prediction models, including \cite{cheng2016wide, guo2017deepfm, lian2018xdeepfm, naumov2019deep, zhou2018deepi, zhou2018deep}, are powered by memory-intensive embedding layers that utterly dwarf the rest of the model. The trade-off between the size of the embedding layer and the statistical performance seems to be an unavoidable trade-off. Generally these deep models are trained via empirical risk minimization (ERM) and back-propagation. For a given model $f_\theta$ (parameterized by $\theta$) the standard practice is to represent categorical features with some indexed embedding layer $E$. The ERM objective is then: $\min_{\theta,E} \sum_{i\in \mathcal{D}} \ell \left(f_\theta(\mathbf{x'}_i,E[(x_1,...,x_\kappa)_i]),y_i\right)$ where the sum is over all data points $\{(x_1,...,x_\kappa,\mathbf{x'})_i,y_i\}$ in the dataset and the loss function $\ell$ is taken to be cross entropy for our purposes. Usually, each categorical feature has its own independent embedding matrix: $E[(x_0,...,x_\kappa)_i] = (E^{(1)}[x_1],..., E^{(\kappa)}[x_\kappa])$.

\textbf{Related Works.} Recent works have proposed similar but substantially different techniques for non-uniform embedding architectures, particularly for the natural language processing (NLP) domain \cite{chen2018groupreduce,baevski2018adaptive}. Neither of those methods would work out-of-the-box for CTR prediction because they ignore the inherit feature-level structure in CTR that is absent in NLP. We discuss key distinctions in more detail in Appendix.

Other approaches propose neural architecture search (NAS) for RecSys embedding layers is proposed in \cite{joglekar2020neural}, where generic reinforcement learning algorithms are used to architect the embedding layer. In contrast to computationally expensive NAS, we show that the architecture search over non-uniform embedding layers can be distilled into tuning a \emph{single} hyper-parameter and does not require the heavy-machinery of NAS. This simplification in model search is only possible due to our theoretical framework. Furthermore, in contrast to all previous works with non-uniform embeddings, we theoretically analyze our method. Moreover, past works do not empirically validate the speculated mechanisms by which their methods work. %

\section{Mixed Dimension Embedding Layer}

The MD embedding layer architecture,  $\mathbf{\bar{E}}$, consists of $k$ blocks and is defined by $2k$ matrices:
$
\mathbf{\bar{E}} = (\bar{E}^{(1)},...,\bar{E}^{(k)}, P^{(1)},...,P^{(k)})
$
with $\bar{E}^{(i)} \in \mathbb{R}^{n_i \times d_i}$ and $P^{(i)} \in \mathbb{R}^{d_i \times \bar{d}}$ for $i=1,...,k$. Together, $\bar{E}^{(i)}$ and $P^{(i)}$ form the $i$-th block. In total, there are $n = \sum_{i=1}^k n_i$ embedding vectors in the layer. We always treat embedding vectors as row vectors. The $\bar{E}^{(i)}$ can be interpreted as the MD embeddings, and the $P^{(i)}$ are projections that lift the embeddings into a \emph{base dimension} $\bar{d}$ such that $\bar{d} \geq d_i$. The entire layer can be thought of as a single $n \times \bar{d}$ block matrix for which the $i$-th block is factored at dimension $d_i$.



\begin{figure}[h!]

    \includegraphics[width=0.47\textwidth]{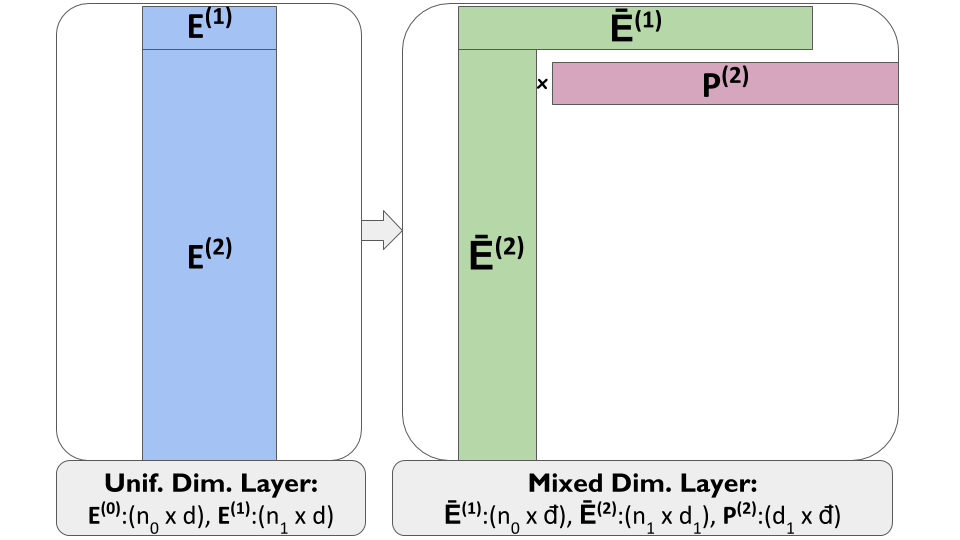}
   \label{fig:dims_and_pop}
\caption{\footnotesize Matrix Architecture for UD and MD Embedding Layers. 
} 
\vspace{-5pt}
\end{figure}

Forward propagation for a MD embedding layer is performed by indexing an embedding vector and then projecting it. For example,  compute $P^{(1)}\bar{E}^{(1)}_{\ell}$ for the $\ell$-th vector in the first block. Downstream models based on a MD embedding layer should be sized with respect to $\bar{d}$. If $d_i = \bar{d}$ for any block, the projection $P^{(i)}$ is not needed and may be replaced with an identity mapping. We illustrate this along with the general matrix architecture of a two block MD embedding layer in Fig. 1.  The parameter budget (total area) consumed by UD and MD embedding layers is the same, but the parameters are allocated unevenly to different indices in the MD embeddings.  

For MD embedding layers, there are two primary architectural decisions to make: (i) \emph{\textbf{blocking:} how to block $n$ total embedding indices into $k$ blocks?} and (ii) \emph{\textbf{sizing:} how to size the embedding dimensions $\mathbf{d} = (d_1,...,d_k)$?} For large-scale CTR, $\kappa$ is generally on the order of 10 to 100. Standard embedding layers allocate $\kappa$ UD embedding matrices to these $\kappa$ features. For MD layers, it is both simple and natural to  inherit the block structure as delineated by the task itself. We let $k = \kappa$ and use the same number of MD embedding blocks as categorical features in the CTR prediction task. The MD layer satisfies $\bar E^{(i)} \in \mathbb{R}^{n_i \times d_i}$ for $i \in \{1,..., \kappa\}$. The value range for each categorical feature defines the row counts $n_i$ in the corresponding block of the MD layer. Any re-indexing can trivially be stored in a low-cost length $k$ offset vector. For the Criteo dataset, there are $\kappa=26$ distinct categorical features, so we produce a MD embedding layer with $k=26$ blocks. To get meaningful and useful blocks, this blocking scheme depends on the fact that our task has a large number of contextual features $k$, with value ranges varying from order 10 to order 10 million. Thus, even if feature values are roughly uniformly popular within each feature, the large variation in value ranges leads to a significantly skewed popularity. In contrast to CTR prediction tasks, when using word embeddings in NLP, one cannot block the mixed layer by feature because this inherent structure is absent. Thus, one needs to resort to complex blocking and sizing schemes, such as those proposed in \cite{chen2018groupreduce,baevski2018adaptive}. Furthermore, we found that accuracy significantly drops if the layer is \emph{not} blocked by feature. We hypothesize that embedding projections encode feature-level semantic information when blocking by feature. As for the question of \emph{sizing} the embedding blocks, we defer  discussion until after our theoretical analysis.

\section{Theoretical framework}


As is standard, our theoretical analysis models CF and RecSys tasks with matrix completion and factorization (additional references and all proofs are in Appendix).

\begin{wrapfigure}{R}{0.25\textwidth}
\begin{minipage}{0.25\textwidth}
\footnotesize
$M = \begin{bmatrix}
M^{(11)}  & \dots  & M^{(1k_W)}  \\
\vdots  & \ddots & \vdots \\
M^{(k_V1)} & \dots & M^{(k_Vk_W)} 
\end{bmatrix}$
\end{minipage}
\end{wrapfigure}


Let $M \in \mathbb{R}^{n \times m}$, for $n \geq m$, be an unknown target matrix. Without loss of generality, we also assume $M$ has a \emph{block structure} such that $M$ is comprised of blocks $M^{(i,j)} \in \mathbb{R}^{n_i \times m_j}$ for $1 \leq i \leq k_V$ and $1 \leq j \leq k_W$. When indexing $M$, we use subscripts, as in $M_{kl}$, to denote the $kl$-th scalar entry in $M$, and superscripts in parenthesis, such as  $M^{(i,j)}$, to denote the $ij$-th block of $M$ (the comma  is often omitted in the superscript). Let $\texttt{rank}(M) = r$ and $\texttt{rank}(M^{(ij)}) = r_{ij}$. Let $\Omega \subset [n] \times [m]$ denote a sample of indices. Our observations, denoted $\mathbf{\Omega}$, act as a training set: $ \mathbf{\Omega} = \{ (k,l,M_{kl}) : (k,l) \in \Omega\}$. We say the target matrix $M$ is \emph{completed} or \emph{recovered} if recovery algorithm $\mathcal{A}$ returns $\mathcal{A}(\mathbf\Omega) = M$. We are interested in the probability of recovery event: $\mathbf{Pr}[M = \mathcal{A}(\mathbf\Omega)]$ for an algorithm $\mathcal{A}$ under a sampling model for $\Omega$. 
 Given both the block-wise structure of $M$ and the MD embeddings, it is straightforward to apply MD. The goal is to train the layer $\mathbf{\bar E}$ to represent $M$ with the block structure in $\mathbf{\bar E}$ inherited from $M$. We can train $\mathbf{\bar E}$ using stochastic gradient descent (SGD).

\textbf{Data-Limited \& Memory-Limited Regimes.} In contextual recommendation engines, there are two primary bottlenecks. In the data-limited regime, (when the number of samples is $o(nr\log n)$) the model does not have enough samples to accurately recover the preference matrix unless a popularity-based approach like MD embeddings is used. In the memory-limited regime (when the space constraint is $o(nr)$), the model has sufficient data to recover the preference matrix but not enough space for the parameters that comprise the embedding layer, which requires us to use fewer parameters than are naively required. We leave analysis of both regimes simultaneously for future work. Because large-scale CTR prediction systems can use up to order $10^9$ contextual features and constantly generate data, they are usually in the memory-limited regime.


\vspace{-3pt}
\subsection{Generalization in the Data-Limited Regime}
\vspace{-3pt}

It is common practice to study a Bernoulli sampling model for $\Omega$ \cite{candes2010power,candes2010matrix,candes2009exact,candes2011tight,sun2016guaranteed}, where each entry is revealed independently with some small probability. Below, we extend Bernoulli sampling for the proposed block-wise structure such that sampling probabilities are constant within a block.

\textbf{Definition: Block-wise Bernoulli sampling} \emph{ \small Let $\Pi \in \mathbb{R}^{k_W \times k_V}$ be a probability matrix. Let $N$ denote the expected number of observed indices in a training sample. Let $\mathbf{i}(\cdot)$ and $\mathbf{j}(\cdot)$ map each index of $M$ to the index of the block to which it belongs. Each index $(k,l)$ is independently sampled as follows:
$ \mathbf{Pr}[ (k,l) \in \Omega ] = N\Pi_{\mathbf{i}(k),\mathbf{j}(l)}/(n_{\mathbf{i}(k)} m_{\mathbf{j}(l)})$ } \vspace{3pt}

We use standard matrix completion assumptions. We show that when the training samples are sufficiently skewed, MD embeddings can recover many more matrices than UD embeddings. We use recovery of the matrix as a proxy for good generalization. For brevity, we focus on exact completion for matrices, but it is well-understood how to extend these results to approximate completion and for low-rank tensors (refer to Appendix).

We refer to any algorithm that ignores popularity as \emph{popularity-agnostic} (formalized in Appendix). Under uniform popularity, popularity-agnostic algorithms need $\Theta( rn \log n)$ samples to recover the matrix \cite{candes2010matrix}. We provide an asymptotic lower bound on the sample complexity for matrix recovery under popularity skew by \emph{any} popularity-agnostic algorithm. 
\begin{theorem} \small \label{thm:mix}  \label{thm:neg} Let $M$ be a target matrix following the block-wise Bernoulli sampling model under probability matrix $\Pi$. Let $\varepsilon = \min\{  \min_i  \frac{1}{n_i}\sum_j \Pi_{ij},\min_j \frac{1}{m_j}\sum_i \Pi_{ij}\}$.

(1) Suppose $N= o(\frac{r}{\varepsilon}  n \log n) $. Then any algorithm that does not take popularity into account will have asymptotically zero probability of recovering $M$. 

(2) Let $C$ be some non-asymptotic constant independent of $n$. If $N \geq C (\max_{ij}\frac{r_{ij}}{\Pi_{ij}}) n \log n$, then mixed dimension factorization with SGD recovers $M$ with high probability.
\end{theorem}

Thm \ref{thm:neg} states that with popularity-agnostic methods, completion of matrix $M$ is bottlenecked by the row or column with lowest popularity. It is impossible to complete rare rows at the same rank as popular rows. When popularity skew is large, the $\frac{1}{\varepsilon}$ factor greatly increases the sample size necessary to complete $M$. In contrast, MD factorization gets a significant reprieve if low-popularity blocks are also treated as low-rank, implying $\max_{ij}\frac{r_{ij}}{\Pi_{ij}} \ll \frac{r}{\varepsilon} $.

\textbf{Two block example.} The theorems above are more easily interpreted for a special case block matrix consisting of two vertically stacked matrices $M = [M^{(1)}, M^{(2)}]^T$. Let us assume block-wise popularity sampling with $\Pi_1 = 1-\epsilon$ and $\Pi_2 = \epsilon$ for small $0 < \epsilon < 1/2$, so that we can interpret $M^{(1)}$ as the popular and $M^{(2)}$ as the rare block. For illustrative purposes, assume that $r_2$ is a small constant and $r_1 \approx \frac{1}{\epsilon}$ is significantly larger. Then popularity-agnostic algorithms suffer a $\frac{1}{\epsilon^{2}}$ quadratic penalty in sample complexity due to popularity skew, whereas MD factorization only pays a $\frac{1}{\epsilon}$ factor because the rare block is completed at much lower rank. 

\vspace{-1pt}
\subsection{Space Efficiency in the Memory-Limited Regime}
\vspace{-3pt}

To study memory-constrained deployment, we assume that we have sufficient data to complete the target matrix. We are instead constrained by a small parameter budget $B$. The goal is to optimally allocate parameters to embedding vectors such that we minimize the expected reconstruction over a non-uniformly sampled test set. Under mild assumptions this dimension allocation problem is a convex program and is amenable to closed-form solutions. We prove that the optimal dimension for a given embedding scales with that embedding's popularity in a manner that depends deeply on the spectral properties of the target matrix.

For most applications, popularity skew is present at test time as well as training time. In this setting, the natural loss metric to study is the \emph{popularity-weighted mean square error (MSE)}.

\textbf{Definition: Popularity-Weighted MSE}
\emph{ \small Let $\Pi$ be a probability matrix. Let $(k,l)$ be a test coordinate sampled according to $\Pi$. The popularity-weighted MSE is given by $L_{\Pi}(M,\hat{M}) = \mathbb{E}_{k,l}|M_{kl} - \hat{M}_{kl}|^2 = \sum_{i,j}\frac{1}{n_im_j}\Pi_{ij}||M^{(ij)} -\hat{M}^{(ij)}||_F^2$.} \vspace{3pt}

Let us now assume that the target matrix is given and that we are trying to optimally size MD embeddings layers, $W$ and $V$, with respect to popularity-weighted MSE reconstruction. We assume to have a small parameter budget, so that we cannot factor the target matrix exactly. We formulate this optimization as a convex program with linear constraints (we treat the dimensions as continuous --- this is a convex relaxation of a hard problem, see Appendix).

\textbf{Convex program for optimizing mixed dimensions:} \small
$$\min_{d_w,d_v} \left( \min_{W,V} L_{\Pi}(M,WV^T) \right) \text{ s.t. } \sum_i n_i(d_w)_i + \sum_j m_j(d_v)_j \leq B$$

\emph{where $d_w \in \mathbb{R}^{k_W}$ and $d_v \in \mathbb{R}^{k_V}$ denote the dimensions of the embedding blocks of $W$ and $V$, respectively.} 
\normalsize

We can obtain a solution using first-order conditions and Lagrange multipliers \cite{luenberger1984linear}. Our analysis reveals that under mild assumptions, the optimal dimension to popularity scaling is the functional inverse of the spectral decay of the target matrix (see Thm. \ref{thm:opt_sol}).

\begin{theorem}
\small
\label{thm:opt_sol}
Let M be a block matrix with block-wise spectral (singular value) decay given by $\sigma_{ij}(\cdot)$. Then, the optimal  embedding dimensions for MD layers W and V are given by:  $(d^*_{w})_i = \sum_{j} d^*_{ij}$, $(d_{v}^*)_j = \sum_{i} d^*_{ij}$, where 
$d_{ij}^* = \sigma_{ij}^{-1}\left( \sqrt{\lambda (n_i + m_j)(n_im_j)\Pi_{ij}^{-1}} \right)$ and $\sum_{ij}(n_i +m_j)d_{ij}^{*} = B$.
\end{theorem}

 When we have a closed-form expression for the spectral decay of $M$, we can give a closed-form solution in terms of that expression. For illustrative purposes, we give the optimal dimensions for the simple two-block example under \emph{power spectral decay}. A matrix with spectral norm $\rho$ exhibits a singular value spectrum with power spectral decay $\beta > 0$ if the $k$-th singular values is given by: $\sigma(k) = \rho k^{-\beta}$. Based on the corollary below, the optimal dimension for an embedding vector scales with its popularity based on a fractional power law.


\begin{corollary}
\small
For a vertically stacked two-block matrix, with each block exhibiting a power spectral decay, then $d_1^* \propto (1-\epsilon)^{\frac{1}{2\beta}}$ and $d_2^* \propto \epsilon^{\frac{1}{2\beta}}$
\end{corollary}


\begin{figure*}[h]
 \begin{center}
  \begin{subfigure}[l]{0.99\textwidth}
  \vspace{-10pt}
    \includegraphics[width=\textwidth]{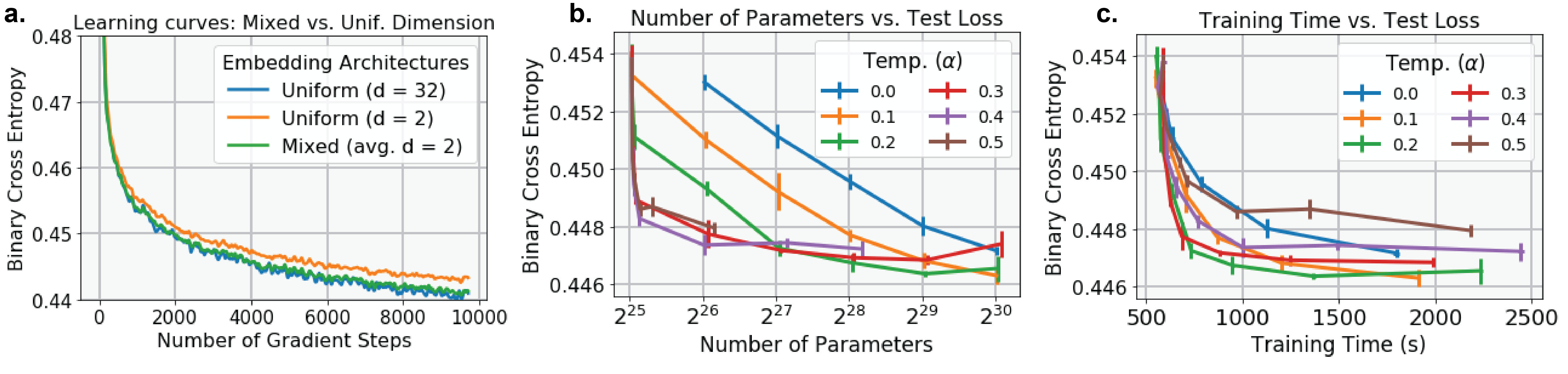}
   \label{fig:ctr_learn}
  \end{subfigure}
   \vspace{-20pt}
\caption{ \footnotesize CTR prediction results for MD embeddings on Criteo dataset using DLRM. Implementation is available as part of an open-source project on GitHub: \texttt{facebookresearch/dlrm}.   Fig. 2a (left): Learning curves for selected emb. arch.  Fig. 2b (center): Loss vs. \# param. for varying $\alpha$. Fig 2c (right): Train time vs. loss for varying $\alpha$ \vspace{-20pt}}
 \end{center}
\end{figure*}

\begin{figure*}[h!]
 \begin{center}
  \begin{subfigure}[l]{0.99\textwidth}
    \includegraphics[width=\textwidth]{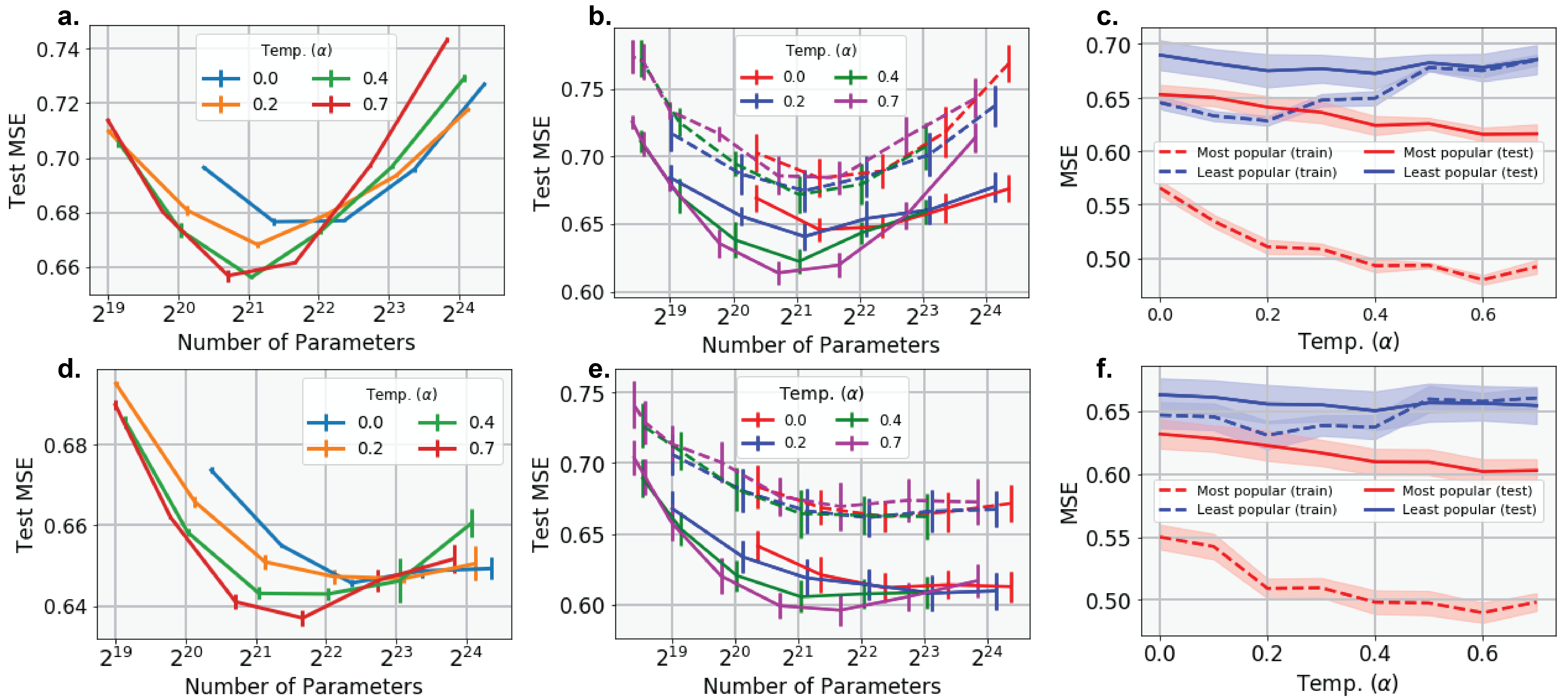}
    \vspace{-10pt}
   \label{fig:mf_perf}
  \end{subfigure}

\caption{\footnotesize Matrix Factorization (top row; a-c) and NCF (bottom row; d-f) for collaborative filtering on the MovieLens dataset. Fig. 3a \& 3d (right): MSE vs. \# params for varying $\alpha$. Fig. 3b \& 3e (center): MSE vs. \# params for varying $\alpha$. Dashed lines correspond to test samples that contain the one-third least popular items. Solid lines correspond to test samples that contain the one-third most popular items. Fig. 3c \& 3f (left): Generalization for popular (red) and rare (blue) items. Dashed lines correspond to training loss and solid lines correspond to test loss. \vspace{-27pt}} 
\label{fig:ml_case_study}
 \end{center}
\end{figure*}

\vspace{-7pt}
\subsection{ Large-scale CTR Prediction is Memory-Limited}  Labeled training data is easy to acquire in most large-scale CTR prediction systems because one can directly observe user engagement (or lack thereof) with personalized content. The embedding layer's memory footprint ends up being the primary bottleneck. In this situation, the results of Thm \ref{thm:opt_sol} yield appropriate guidelines for the optimal dimension for each embedding vector. The unavoidable difficulty is that one needs to know the spectrum of the target matrix to know the optimal dimension. One solution is to train an enormous embedding table with an enormous data set, thereby obtaining the spectrum, and then factoring (or re-training from scratch) at the optimal size. However, this solution still requires enormous resource usage during training. Alternatively, we can still leverage the insight of our theoretical framework to efficiently find good mixed embedding dimensions. Most spectral decays, whether they are flat or steep, can be decently well fit by a power law (so much so that there is a large literature dedicated to \emph{falsifying} power-laws even when the numerical fit appears reasonable \cite{clauset2009power}). Varying the temperature adequately captures the trend of the decay. By \emph{a priori} assuming that the spectral decay is a power law, we only have to tune one hyper-parameter over a narrow range that requires only exploring a small number of values. 



\textbf{Power-law Sizing Scheme.} Here we define \emph{block-level probability} $\mathbf{p}$. Let $\mathbf{p}$ be a $k$-dimensional probability vector (recall $k$ is the \# of blocks in the embedding layer) such that $\mathbf{p}_i$ is the average query probability over rows in the $i$-th block. When blocks exactly correspond to features, as in our CTR experiments, then $\mathbf{p}_i = \frac{1}{n_i}$ because each one row per block is queried per inference based on the value of each feature. More generally, under a block sampling structure $\Pi$, $\mathbf{p}_i = \sum_j \Pi_{ij} $.

\vspace{-7pt}
\begin{algorithm}[H]
\caption{Popularity-Based Dimension Sizing}
\begin{algorithmic}
\footnotesize
\State \textbf{Input:} Baseline dimension $\bar{d}$ and 
fixed temperature $0 \leq \alpha \leq 1$
\State \textbf{Input:} Probability vector $\textbf{p}$
\State \textbf{Output:} Dimension assignment vector  $\textbf{d}$
\State $\lambda \gets \bar{d} ||\mathbf{p}||_\infty^{-\alpha}$ \Comment{Compute scalar scaling factor}
\State $\textbf{d} \gets \lambda \mathbf{p}^\alpha$ \Comment{ Component-wise exponent}
\end{algorithmic}
\label{alg:alpha}
\end{algorithm}
\vspace{-7pt}

Alg. \ref{alg:alpha} shows code for the scheme. We use a temperature parameter $\alpha > 0$ to control the degree to which popularity influences the embedding dimension (as a simplification over using decay parameter $\beta$).  As $\alpha$ increases, so does the popularity-based skew in the embedding dimension. At $\alpha=0$, the embedding dimensions are all uniform. At $\alpha=1$, the embedding dimensions are proportional to their popularity.  




\section{Experiments}





We measure memory in terms of the number of $32$-bit floating point parameters in the model. Since different embedding base dimensions imply different widths in first hidden layers of the downstream deep model, for fairness, our parameter counts include both the embedding layer and all model parameters (recall that the parameter count is overwhelmingly dominated by embeddings). We report statistical performance in terms of cross entropy loss. Accuracy is reported in Appendix  (along with other experimental details). DLRM with uniform $d=32$ embeddings obtains an accuracy of $\sim79$\%, close to state-of-the-art for this dataset \cite{naumov2019deep}. We sweep parameter budgets from a lower bound given by 1 parameter per embedding vector ($d=1$) to an upper bound given by the memory limit of a $16$\texttt{GB} GPU. The parameters are allocated to embeddings according to the $\alpha$-parameterized rule proposed in Alg. \ref{alg:alpha}.

MD embeddings with $\alpha = 0.3$ produce a learning curve on par to that of $d=32$ UD embeddings  using a total parameter count equivalent to $d=2$ UD  (Fig. 2a), yielding a $16\times$ reduction in parameters. We can see that using MD embedding layers improves the memory-performance frontier at each parameter budget (Fig. 2b). The optimal temperature ($\alpha$) is dependent on the parameter budget, with higher temperatures leading to lower loss for smaller budgets. Embeddings with $\alpha = 0.4$ obtain performance on par with uniform embeddings using $16\times$ fewer parameters. Embeddings with $\alpha=0.2$ modestly outperform UD embeddings by an accuracy of $0.1$\% using half as many parameters. For RecSys, small improvements in accuracy are still important when volume is sufficiently large. The std. dev. estimates indicate that this gain is significant. MD embeddings not only improve prediction quality for a given parameter budget, but also for a given training time. MD embeddings can train $>2\times$ faster than UD embeddings at a given test loss  (Fig. 2c). This is possible because for a given loss, the MD layer uses far fewer parameters than a UD layer. The faster training we observe is likely due to superior bandwidth efficiency and caching on the GPU, enabled by the reduced memory footprint. We run all of our models on a single GPU with identical hyperparameters (including batch size) across all $\alpha$ (more details in Appendix).

\vspace{-1pt}
\textbf{Case Study: Collaborative Filtering.} Because context-free CF only includes two features, users and items, it is easier to gain insight into the effects of MD embeddings.  Matrix factorization (MF) is a typical algorithm for CF where the matrix factors are embeddings.  We also include Neural Collaborative Filtering (NCF) models \cite{he2017neural} to ensure that trends still hold when non-linear layers are introduced. Because we only have two features for this case study (users and items) we slightly modify our blocking approach. We sort users and items by empirical popularity, and uniformly partition them into block matrices. Then we can apply mixed dimensions within the users and items based on partitions (more details in Appendix).

\vspace{-2pt}
MD embeddings substantially outperform UD embeddings, especially at low parameter budgets (Figs. 3a \& 3d). For Figs. 3b \& 3e, we partition \emph{item} embeddings (since they often the focal point of recommendation) into the one-third most and least popular items (by empirical training frequency).  These results show that non-zero temperature ($\alpha$)  improves test loss on popular items and performs on par with or slightly better than UD embeddings on rare items. We report generalization curves for the popular and rare items in Figs 3c and 3f. We fix a parameter budget ($2^{21}$),  vary the temperature ($\alpha$), and plot training and test loss for both popular and rare item partitions. Allocating more parameters to popular items by increasing temperature ($\alpha$) decreases both training and test loss. On the other hand, allocating more parameters to rare items by decreasing temperature ($\alpha$) only decreases training loss but not test loss, indicating that the additional parameters on the rare embeddings are wasteful. Uniform parameter allocation, agnostic to popularity, is inefficient. A majority of parameters are severely underutilized on rare embeddings, whereas popular embeddings could still benefit from increased representational capacity.


\vspace{-6pt}
\section{Conclusion}
\vspace{-3pt}
We show that MD embeddings greatly improve both parameter efficiency and training time. Through our case study, we demonstrate systematic and compelling empirical evidence that MD embeddings work by improving capacity for learning popular embeddings without compromising rare embeddings. Our theoretical framework is the first to mathematically explain how this occurs, and our experiments are the first to validate the phenomena.

\section*{Acknowledgments}

This work was initiated while A.A.G. was an intern at Facebook. The authors would like to thank Shubho Sengupta, Michael Shi, Jongsoo Park, Jonathan Frankle and Misha Smelyanskiy for helpful comments and suggestions about this work, and Abigail Grosskopf for editorial help with the manuscript.


\bibliographystyle{unsrt}
\bibliography{bib}



\appendix

The appendix is organized as follows. We begin with an extended discussion, including additional related works and background references in A. We include supplementary experiments and reproducibility details concerning our empirical work in B. Finally, we include mathematical details including assumptions, theorems, and proofs in C.

\section*{Supplementary \& Extended Discussion}

\subsection{Representation Learning} Embedding-based approaches, such as matrix factorization (MF) \cite{hastie2015matrix, Koren2009, Rendle2019} or neural collaborative filtering (NCF) \cite{Dacrema2019,he2017neural,xue2017deep,fan2018matrix}, are among the most computationally efficient solutions to CF and matrix completion. The simplest model for canonical CF is MF, which treats embeddings as factorizing the target matrix directly and comes with nice theoretical guarantees. Direct optimization over embedding layers is a non-convex problem. However, due to specific problem structure, many simple algorithms, including first-order methods, have been proven to find global optima (under mild assumptions) \cite{hardt2014understanding,sun2016guaranteed,jain2013low}.  Neural collaborative filtering (NCF) models, which make use of non-linear neural processing, are more flexible options. While NCF models often lack the theoretical guarantees offered by MF, they usually perform mildly better on real-world datasets. In CTR prediction, it is common to use embedding layers to represent categorical features and have various neural layers to process the various embeddings.

\subsection{Regularization} We note that in many instances, embeddings for collaborative filtering tasks are usually trained with some type of norm-based regularization for the embedding vectors. While this particular form of regularization works well for small-scale embeddings, it is non-trivial to scale it for large-scale embeddings. For many standard loss functions, only the embeddings queried on the forward pass have non-zero gradients on the backward pass. Using sparse updates for the embedding tables is essentially mandatory for efficient training at large scale. Thus, contrary to popular belief in academic circles, large-scale embeddings are often trained without norm-based regularization which are incompatible with sparse updates. This is because the gradient update when using a regularization term should back-propagate to every embedding vector in the layer, rather than just those queried in the forward pass. Because this technique is not feasible in large-scale CTR prediction, we explicitly do not use embedding regularization in our collaborative filtering task. Furthermore, we note that from the perspective of popularity skew, that embedding norm regularization actually implicitly penalizes rare embeddings more than popular ones, since a larger fraction of training updates only contain norm penalty signal for rare embeddings than popular ones. This is an interesting connection that could be explored in future work, but it does not achieve the stated goal of parameter reduction.

\subsection{Non-uniform Embeddings} Recent works have proposed similar but distinct non-uniform embedding architectures, particularly for the natural language processing (NLP) domain \cite{chen2018groupreduce,baevski2018adaptive}. As mentioned in the main text, there are substantial differences between those methods and ours. We emphasize that neither of those methods would work out-of-the-box for CTR prediction. Critically, in contrast to NLP, CTR embeddings encode categorical values for individual features, and thus come with feature-level structure that should not be ignored when architecting the embedding layer. In NLP, embeddings represent words and can be thought of a single large bag of values --- in contrast to representing the various categorical values a particular feature can take on. We discover that ignoring this feature-level structure in the embedding layer adversely affects performance, dropping accuracy by $>1\%$. For this reason, the sorted blocking technique introduced in \cite{baevski2018adaptive} is not effective in CTR prediction. Additionally, embedding layers in NLP are pre-trained from unsupervised language corpus, unlike in RecSys, which means that clustering and factoring the embeddings as in \cite{chen2018groupreduce} prior to training is not feasible in CTR prediction.

Furthermore, in contrast to previous works, we theoretically analyze our method. Moreover, past methods do not even empirically validate the speculated mechanisms by which their methods work. For example, in \cite{baevski2018adaptive}, authors claim their proposed architecture, "reduces the capacity for less frequent words with the benefit of reducing overfitting to rare word." While the proposed method works well on benchmarks, the claim that the method reduces overfitting is not supported. As shown in \cite{yin2018dimensionality,arora2019implicit}, embedding overfitting depends critically on the training algorithm and even the model atop the embeddings. In fact, when training is properly tuned, embeddings are quite resilient to overfitting, which means the claim made in \cite{baevski2018adaptive} is far from self-evident. It is more accurate to view rare embeddings as wasting parameters rather than overfitting them.

Non-uniform and deterministic sampling have been discussed in the matrix completion literature \cite{negahban2012restricted,meka2009matrix, liu2017new}, but only in so far as how to correct for popularity so as to improve statistical recovery performance, or build theoretical guarantees for completion under deterministic or non-uniform sampling.

\subsection{Collaborative Filtering \& Matrix Completion}

 CF tasks come in many variations and have a large mass scientific literature, with good reviews of classical algorithms and approaches provided in \cite{adomavicius2005toward,lousame2009taxonomy}. Related to CF is the simplified matrix completion problem \cite{candes2010matrix, candes2009exact, candes2010power, hastie2015matrix, keshavan2010matrix,vishwanath2010information, sun2016guaranteed, recht2011simpler}.

\subsection{Memory-Efficient Embeddings} The techniques to decrease the memory consumed by embedding tables can be roughly split into two high-level classes: (i) compression algorithms and (ii) compressed architectures. Simple \emph{offline} compression algorithms include post-training quantization, pruning or low-rank SVD \cite{Andrews2015, Bhavana2019, Sattigeri, Sun2016}.  \textit{Online} compression algorithms, including quantization-aware training, gradual pruning, and periodic regularization, are somewhat less popular in practice because of the additional complications they add to already intricate deep model training methods. \cite{alvarez2017compression,frankle2018lottery,Naumov2018reg,park2018value,kang2020learning}. Model distillation techniques \cite{Shu2017,Tissier2018,serra2017getting,tang2018ranking} are another form of compression and can have online and offline variants. On the other hand, compressed architectures have the advantage of not only reducing memory requirements for inference time, but also at training time. This is the approach followed by hashing-based and tensor factorization methods \cite{attenberg2009collaborative, gao2018cuckoo, karatzoglou2010collaborative, Khrulkov2019, Shi2018}.

\subsection{Forward Propagation with Mixed Dimension Embeddings}

Forward propagation for a MD embedding layer can be summarized in Alg. \ref{alg:emb_layer}. The steps involved in this algorithm are differentiable, therefore we can perform backward propagation through this layer and update matrices $\bar{E}^{(i)}$ and $P^{(i)}$ accordingly. We note that Alg. \ref{alg:emb_layer} may be generalized to support multi-hot lookups, where embedding vectors corresponding to $z$ query indices are fetched and reduced by a differentiable operator, such as add, multiply or concatenation.

\begin{algorithm}[h]
\caption{Forward Propagation}
\begin{algorithmic}
\footnotesize
   \State \textbf{Input:} Index $x \in [n]$
   \State \textbf{Output:} Embedding vector $\textbf{e}_{x}$ 
     \State $i \gets 0$ and $t \gets 0$ 
     \While{$t + n_i < x$} \Comment{Find offset $t$ and sub-block $i$}
     \State $t \gets t + n_i$
       \State $i \gets i + 1$ 
     \EndWhile 
     \State $\textbf{e}_{x} \gets \bar{E}^{(i)}[x-t]P^{(i)}$  \Comment{Construct an embedding vector}
\end{algorithmic}
\label{alg:emb_layer}
\end{algorithm}

\subsection{Popularity Histograms}

The key idea of the MD embeddings is to address the skew in the popularity of objects appearing in a dataset. To illustrate it we present popularity histograms of accesses across all features for the Criteo Kaggle dataset in Fig. \ref{fig:criteo_hist} and individually for users and items features for the MovieLens dataset in Fig. \ref{fig:user_hist} and \ref{fig:items_hist}, respectively. 

\begin{figure*}[h]
 \begin{center}
  \begin{subfigure}[l]{0.32\textwidth}
    \includegraphics[width=\textwidth]{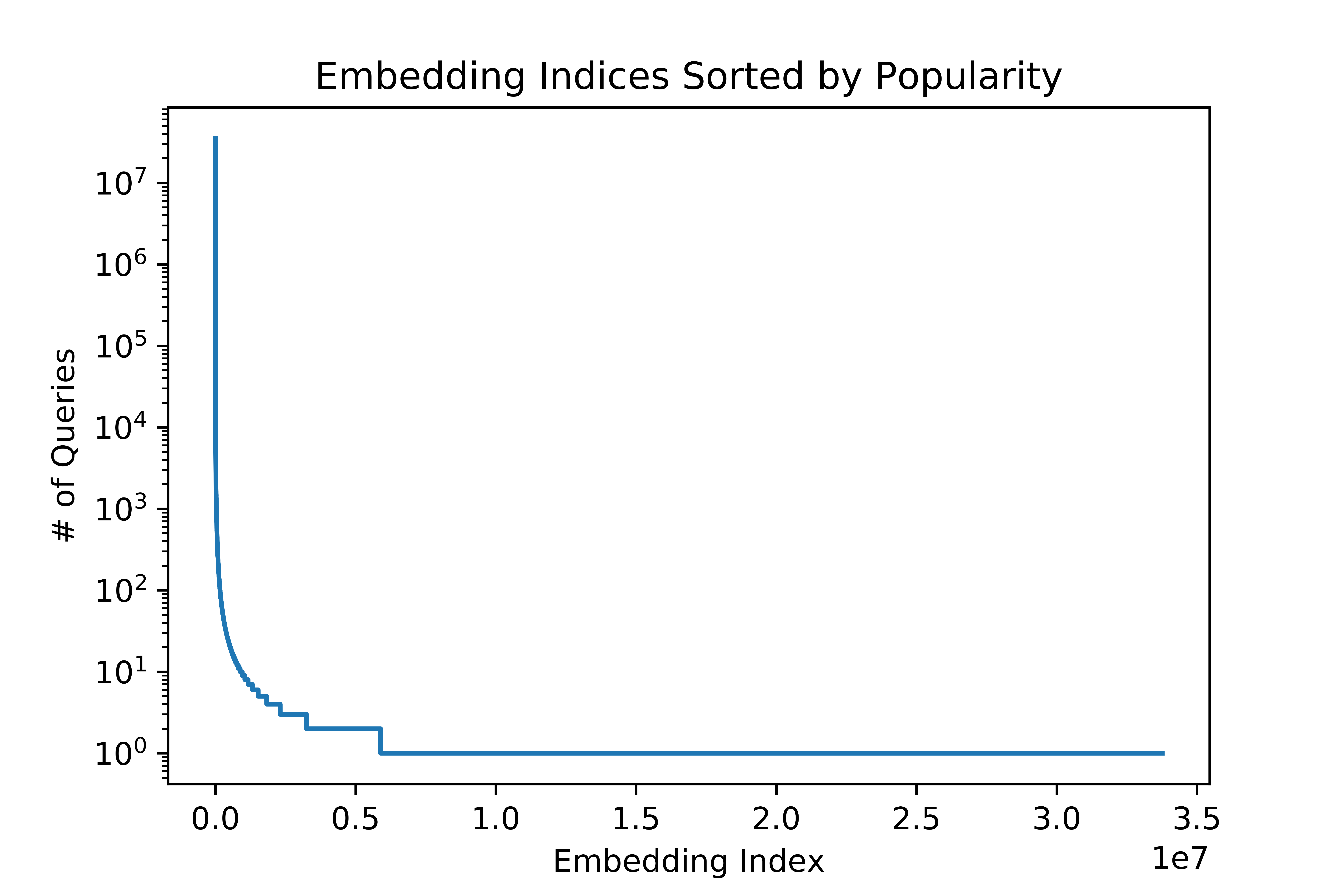}
    \caption{Histogram of accesses for Criteo}
      \label{fig:criteo_hist}
  \end{subfigure}
  \begin{subfigure}[c]{0.32\textwidth} 
   \includegraphics[width=\textwidth]{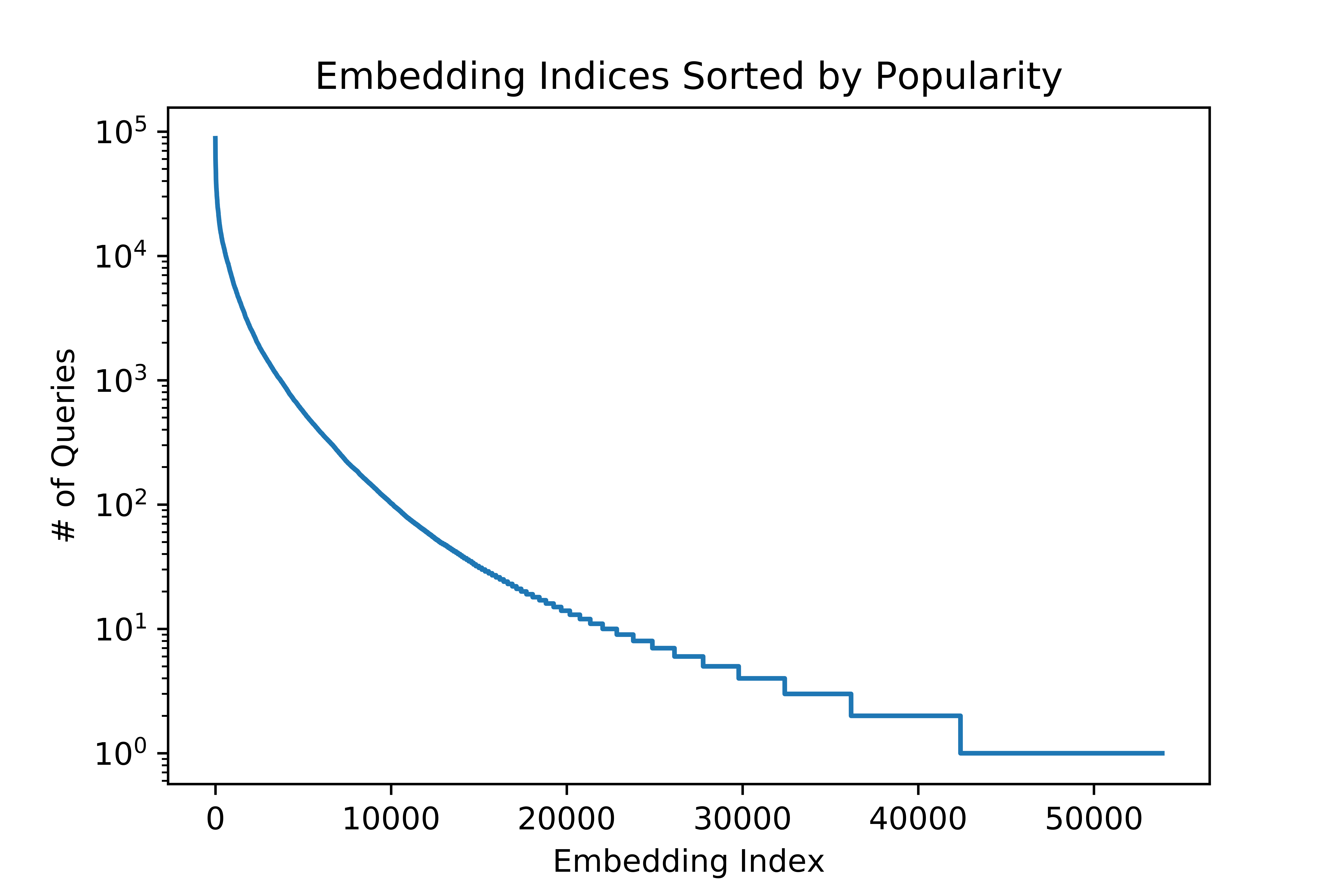}
   \caption{Histogram of accesses for items}
       \label{fig:items_hist}

  \end{subfigure} 
  \begin{subfigure}[r]{0.32\textwidth}
\includegraphics[width=\textwidth]{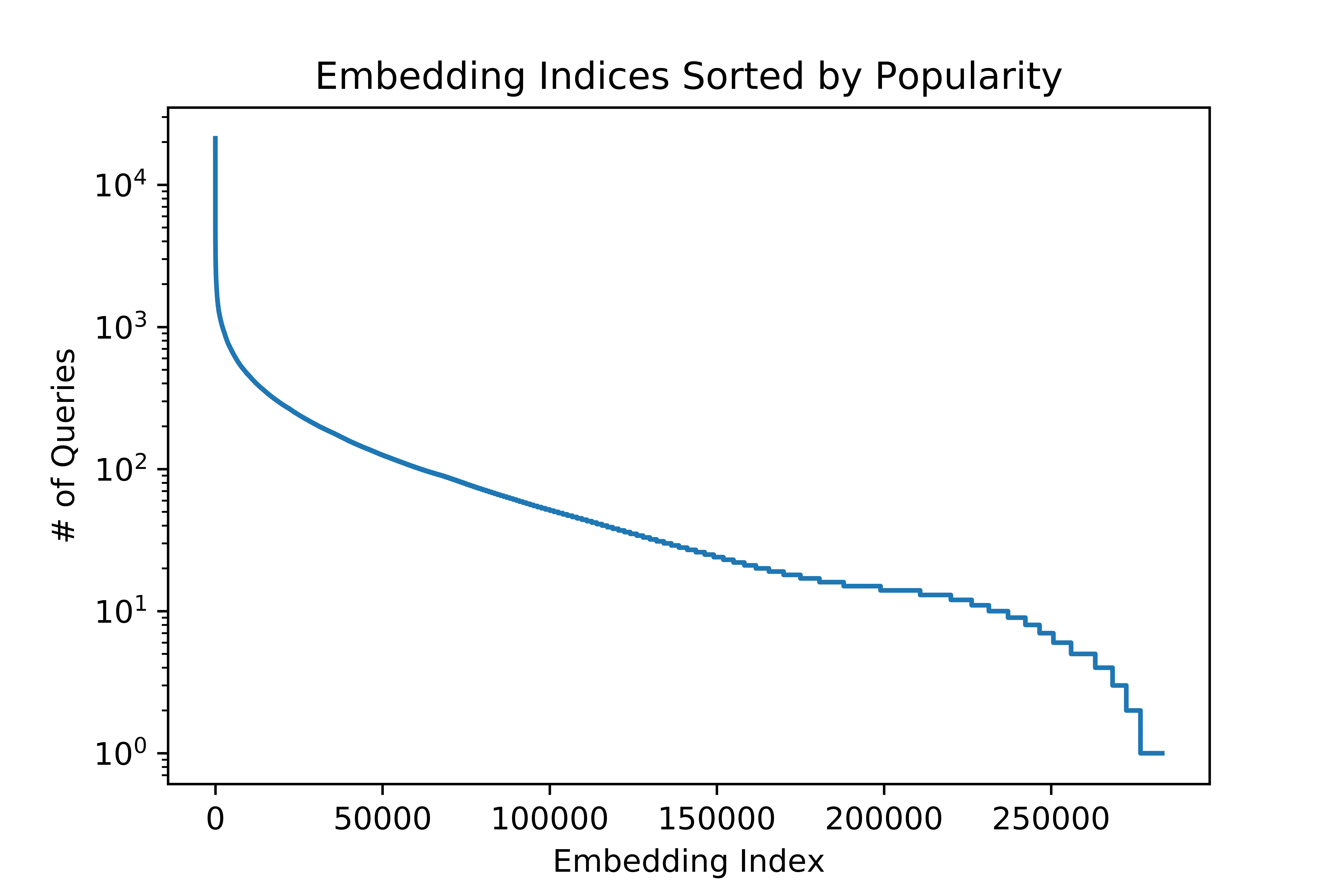}
   \caption{Histogram of accesses for users}
       \label{fig:user_hist}

  \end{subfigure}
\caption{Popularity skew in real-world datasets.}
 \end{center}
\end{figure*}

\begin{figure*}[h]
 \begin{center}
  \begin{subfigure}[l]{0.4\textwidth}
    \includegraphics[width=\textwidth]{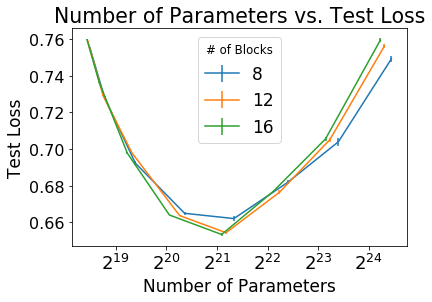}
    \caption{\small Effect of num. of equiparititons on CF with MD}
  \end{subfigure}
  \begin{subfigure}[c]{0.4\textwidth} 
   \includegraphics[width=\textwidth]{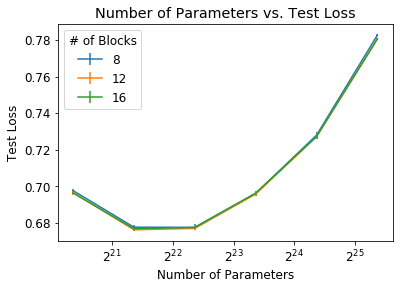}
   \caption{\small Effect of num. of equiparititons on CF with UD}
  \end{subfigure} 
    \caption{\small Effect of num. of equiparititons on CF}
 \end{center}

\end{figure*}

\section*{Experimental Details \& Supplementary Experiments}

In this section we provide detailed protocols for the experiments in Section 5 of the main text. All code is implemented using the Pytorch \cite{paszke2019pytorch} library. All algorithms are run as single GPU jobs on Nvidia V100s \cite{NvidiaVolta100}. All confidence intervals reported are standard deviations obtained from 5 replicates per experiment with different random seeds. In Table \ref{table:dataset} we summarize the datasets and models used.

\begin{table*}[h]
  \centering
  \begin{tabular}{||c c c c c||} 
      \hline
      Dataset & Tasks & \# Samples & \# Categories & Models Used  \\ [0.25ex] 
      \hline\hline
      MovieLens & CF & 27M & 330K & MF, NCF \\ 
      \hline
      Criteo Kaggle & CTR & 40M & 32M & DLRM \\
      \hline
  \end{tabular}
  \caption{Datasets, tasks and models}
  \label{table:dataset}
  \vspace{-10pt}
\end{table*}

\subsection{Collaborative Filtering Case Study}

Recall that CF tasks require a different blocking scheme than the one presented in the main text because we only have 2 categorical features. These features have corresponding embedding matrices $W \in \mathbb{R}^{n \times d}$ and $V \in \mathbb{R}^{m \times d}$, for users and items, respectively. To size the MD embedding layer we apply MDs within individual embedding matrices by partitioning them. We block $W$ and $V$ separately. First, we sort and re-index the rows based on empirical row-wise frequency: $i < i' \implies f_i \geq f_{i'}$, where $f_i$ is the frequency that a user or item appears in the training data\footnote{Sorting and indexing can be done quickly on a single node as well as in the distributed settings.}. Then, we partition each embedding matrix into $k$ blocks such that the sum of frequencies in each block is equal. For each of the $k$ blocks, the total popularity (i.e. probability that a random query will include that row) for each block is constant (see Alg. \ref{alg:equipart}). Then, for a given frequency $\mathbf{f}$ the $k$-equipartition is unique and is simple to compute.  In our experiments, we saw that setting $k$ anywhere in the $[8,16]$ range is sufficient to observe the effects induced by MDs, with diminishing effect beyond these thresholds (Fig. 5).

\vspace{-3pt}
\begin{algorithm}[h]
\caption{Partitioning Scheme for CF}
\begin{algorithmic}
\footnotesize
\State \textbf{Input:} Desired number of blocks $k$ 
\State \textbf{Input:} Row-wise frequencies vector $\mathbf{\bar{f}}$ 
\State \textbf{Output:} Offsets vector $\mathbf{t}$
\State $\mathbf{f} \gets \texttt{sort} (\mathbf{\bar{f}})$ \Comment{Sort by row-wise frequency}
\State Find offsets $t_i$ such that $\sum_{l=t_{i}}^{t_{i+1}} f_l = \sum_{l=t_{j}}^{t_{j+1}} f_l$ for $\forall i,j$ 
\end{algorithmic}
\label{alg:equipart}
\end{algorithm}

In our experiments we use the full $27$M MovieLens dataset \cite{Harper2015}. We train at learning rate $10^{-2}$, found in the same manner as for CTR prediction. For consistency with CTR prediction, we also used the Amsgrad optimizer \cite{Reddi2018}. We train for 100 epochs, taking the model with the lowest validation loss. We early terminate if the model does not improve after 5 epochs in a row. We use a batch size of $2^{15}$ in order to speed-up training. We did not observe significant differences between this and smaller batch sizes. Our reported performance, in terms of MSE, are comparable to those elsewhere reported in the literature \cite{strub2016hybrid}.

For initialization we use the uniform Xavier initialization (for consistency with CTR prediction). Also, for the NCF model, we use a 3-layer MLP with hidden layer dimensions $128-128-32$. We used default LeakyReLU activations.

For the item embedding partitioning (in Fig. \ref{fig:ml_case_study}), we partitioned the item embeddings by empirical popularity mass. This means that the top third item embeddings represent to order $10^3$ items, whereas the bottom third item embeddings represent order $10^5$ items. Thus, the top third and bottom third partitions have the same empirical popularity mass, not the same number of items.

\subsection{CTR Prediction}

From the perspective of this work, using MD embedding layers on real CTR prediction data with modern deep recommendation models is an important experiment that shows how MD embedding layers might scale to real-world recommendation engines.

\begin{figure}[h]
\small
 \centering
   \includegraphics[width=0.35\textwidth]{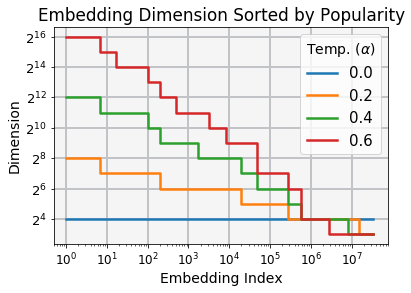}
   \caption{ \small Embedding parameters for 26
categorical features allocated at different temperatures ($\alpha$) for the same parameter budget on the Criteo dataset. Higher temperatures
results in higher dimensions for popular embeddings and $\alpha=0$ is uniform dimensions.See Section 4 and Alg.1 for more details concerning the assignment scheme. The dimensions are rounded to powers of 2.}
   \label{fig:emb_dims_sorted}
\end{figure}

In our experiments we use state-of-the-art DLRM \cite{naumov2019deep} and the Criteo Kaggle dataset. We determined the learning rates, optimizer, batch size, and initializations scheme by doing a simple grid search sweep on \emph{uniform} dimension embeddings of dimension $32$. \emph{Ultimately, we used Amsgrad with a learning rate of $10^{-3}$, a batch size of $2^{12}$, and a uniform Xavier initialization for all CTR experiments reported in the main text}. As is customary in CTR prediction tasks, we train for only one epoch \cite{naumov2019deep}. Examples of parameter allocation based on Alg. \ref{alg:alpha} can be found in Fig. \ref{fig:emb_dims_sorted}.

\begin{figure}[h!]
\small
 \centering
   \includegraphics[width=0.35\textwidth]{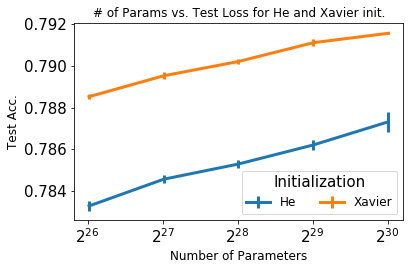}
   \caption{ \small Xavier and He (fan-out) initializations for DLRM with UD embeddings at various parameter budgets}
   \label{fig:init_ctr}
\end{figure}

\begin{figure}[h!]
\small
 \centering
   \includegraphics[width=0.35\textwidth]{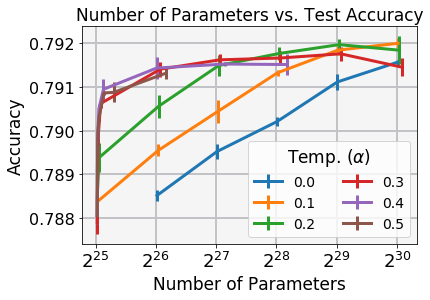}
   \caption{\small Accuracy on CTR Prediction for varying temp. ($\alpha$)}
   \label{fig:acc_ctr}
\end{figure}

\begin{figure}[h!]
 \centering
   \includegraphics[width=0.35\textwidth]{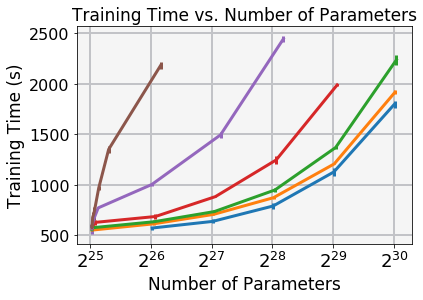}

   \caption{\small Training time vs. parameter counts for varying temperature (legend same as Fig. \ref{fig:acc_ctr})}
   \label{fig:tt}
\end{figure}

 For learning rates, we tried $\lambda \in \{10^{-2},10^{-3},10^{-4},10^{-5}\}$. $10^{-3}$ was best. We tried Amsgrad, Adam, Adagrad, and SGD optimizers. Amsgrad was best.

For batch size, we tried powers of 2 from $2^{5}$ to $2^{12}$. The batch size hardly affected the memory footprint on the GPU, since the majority of memory is spent on the embedding layer, which is only sparsely queried on each forward pass in the batch. We also found that performance was largely invariant in batch size, so we selected batch size of $2^{12}$.

For initialization schemes, we tried uniform Xavier \cite{glorot2010understanding}, uniform He  (fan-in) and uniform He (fan-out) \cite{he2015delving}. We initialized all neural network parameters, including embedding matrices, according to the same scheme. We found that He fan-in resulted in severe training instability. Xavier outperformed He fan-out by a considerable margin (Fig. \ref{fig:init_ctr}).

We also report the accuracy for our CTR prediction experiments (Fig. \ref{fig:acc_ctr}). The curves are like a reflection of the cross entropy loss in the main text.

Finally, we report training time vs. parameter counts for varying temperatures (Fig. \ref{fig:tt}).

\subsection{On the Range of Temperatures ($\alpha$)}
Recall that we use a power-law inspired embedding dimension sizing scheme. In the main text, we emphasize that $\alpha$ should be in the range $(0,1)$. In principle, one could pick an $\alpha >1$, but since it is natural to assume that there is diminishing returns to the embedding dimension of a feature, it should follow that such a choice is poor. An $\alpha >1/2$ would imply a sub-linear spectral decay which is rarely the case in embeddings learned form real-world data. This coincides with out experiments, where the best $\alpha$ were actually below $1/2$.

\section*{Theoretical Details}
We proceed to present proofs of theorems as well as additional results that did not fit into the main text of the paper. First, we describe the details of the mathematical optimization procedure used in the proofs. Then, we discuss the extension of our method from the matrix case to the tensor case. Finally, we enumerate the details/assumptions and give the proofs for our theorems.

\subsection{Block-wise MD Factorization}

The first point to address is the specifics of the SGD-variant used to solve the MD factorization. We adopt the particular variation of the SGD algorithm for matrix factorization proposed in \cite{sun2016guaranteed} and assume that block structure is known or has been estimated. We show that under rank additivity assumption it can be applied to factor the blocks of the matrix $M$ independently and yet construct a rank-$r$ approximation for it. Note that in this scenario the projections do not need to be free parameters (see Alg. \ref{alg:blockwise_mixd_fact}).

\begin{algorithm}[h]
\small
\caption{Block-wise Mixed Dim. Factorization}
\begin{algorithmic}
\small
\State \textbf{Input:} Partially masked target block matrix $\mathcal{P}_\Omega(M)$. 
\State \textbf{Input:} The blocks $M^{(ij)}$ with block-wise rank $r_{ij}$. 
\State \textbf{Output:} MD embeddings $\bar{W},\bar{V}$. 
\State $W^{(i,j)}, V^{(i,j)} \gets \mathbf{SGD}(\mathcal{P}_\Omega(M^{(ij)}))$\Comment{Factor each block}
\For{ $1 \leq i \leq k_W$ }
\State $\bar{W}^{(i)} \gets [W^{(i,1)},...,W^{(i,k_V)}] \in \mathbb{R}^{n_i \times d_w^i}$ \Comment{Assemble $\bar{W}$}
\State $d_w^{i} \gets \sum_{j=1}^{k_V}r_{ij}$ \Comment{Construct projection $P$}
\State $s_w^{i} \gets \sum_{l=1}^{i-1} \sum_{j=1}^{k_V} r_{lj}$
\State $t_w^{i} \gets \sum_{l=i+1}^{k_W} \sum_{j=1}^{k_V} r_{lj}$
\State $P_W^{(i)} \gets \begin{bmatrix} 0_{d_w^i \times s_w^i} , I_{d_w^i \times d_w^i}, 0_{d_w^i \times t_w^i} \end{bmatrix} \in \mathbb{R}^{d_w^i \times r}$
\EndFor

\For{ $1 \leq j \leq k_V$ }
\State $\bar{V}^{(j)} \gets [V^{(1,j)},...,V^{(k_W,j)}] \in \mathbb{R}^{m_j \times d_v^j}$ \Comment{Assemble $\bar{V}$}
\State $d_v^j \gets \sum_{i=1}^{k_W}r_{ij}$ \Comment{Construct projection $P$}
\State $s_{ij} \gets \sum_{l=1}^{i-1}r_{li} + \sum_{l=1}^{j-1}r_{il}$
\State $t_{ij} \gets \sum_{l=j+1}^{k_V}r_{il} + \sum_{l=i+1}^{k_W}r_{lj}$
\State $P_V^{(j)} \gets \begin{bmatrix}
0_{r_{1j} \times s_{1j}}, I_{r_{1j} \times r_{1j}}, 0_{r_{1j} \times t_{1j}}\\
\vdots\\
0_{r_{ij} \times s_{ij}} , I_{r_{ij} \times r_{ij}}, 0_{r_{ij} \times t_{ij}} \\
\vdots\\
0_{r_{k_wj} \times s_{k_wj}} , I_{r_{k_wj} \times r_{k_wj}}, 0_{r_{k_wj} \times t_{k_wj}}
\end{bmatrix}  \in \mathbb{R}^{d_v^j \times r}$
\EndFor

\State$\bar{W} \gets (\bar{W}^{(1)},...,\bar{W}^{(k_W)},P_W^{(1)},...,P_W^{(k_W)})$
\State$\bar{V} \gets (\bar{V}^{(1)},...,\bar{V}^{(k_V)},P_V^{(1)},...,P_V^{(k_V)})$

\end{algorithmic}
\label{alg:blockwise_mixd_fact}
\end{algorithm}

\paragraph{Two block example} We illustrate block-wise MD factorization on $2n \times m$ two block rank-$r$ matrix 

$M 
{=} 
\begin{bmatrix}
M^{(1)} \\
M^{(2)}
\end{bmatrix}
{=} 
\begin{bmatrix}
W^{(1)}V^{(1)^T} \\
W^{(2)}V^{(2)^T}
\end{bmatrix}
= 
\begin{bmatrix}
W^{(1)} P_W^{(1)} \\
W^{(2)} P_W^{(2)}
\end{bmatrix}
\begin{bmatrix}
V^{(1)^T} \\
V^{(2)^T}
\end{bmatrix}$
\text{with}
$W^{(1)} \in \mathbb{R}^{n \times r_1}, V^{(1)} \in  \mathbb{R}^{m \times r_1}, W^{(2)} \in \mathbb{R}^{n \times r_2}, V^{(2)} \in  \mathbb{R}^{m \times r_2}$ obtained by Alg. \ref{alg:blockwise_mixd_fact}, projections $P_W^{(1)} = [I, 0] \in \mathbb{R}^{r_1 \times r}$, $P_W^{(2)} = [0, I] \in \mathbb{R}^{r_2 \times r}$ defined by construction and $I$ being an identity matrix.

\paragraph{Four block example} We extend the example to $2n \times 2m$ four block rank-$r$ matrix below

$M = \begin{bmatrix}
M^{(11)} M^{(12)}\\
M^{(21)} M^{(22)}\\
\end{bmatrix}
=
\begin{bmatrix}
\bar{W}^{(1)} P_W^{(1)} \\
\bar{W}^{(2)} P_W^{(2)}
\end{bmatrix}
\begin{bmatrix}
(\bar{V}^{(1)} P_V^{(1)} )^T \\
(\bar{V}^{(2)} P_V^{(2)} )^T
\end{bmatrix}$

where rank-$r_{ij}$ block $M^{(ij)} = W^{(ij)}V^{(ij)^T}$factors and   

$$\bar{W}^{(1)} = [W^{(11)}, W^{(12)}] \in \mathbb{R}^{n \times r_{11} + r_{12} }$$
$$\bar{W}^{(2)} = [W^{(21)}, W^{(22)}] \in \mathbb{R}^{n \times r_{21} + r_{22} }$$
$$\bar{V}^{(1)} = [V^{(11)},\phantom{1.} V^{(21)}] \in \mathbb{R}^{n \times r_{11} + r_{21} }$$
$$\bar{V}^{(2)} = [V^{(12)},\phantom{1.} V^{(22)}] \in \mathbb{R}^{n \times r_{12} +  r_{22} }$$


were obtained by Alg. \ref{alg:blockwise_mixd_fact}, while projections

$P_W^{(1)} = \begin{bmatrix} I,0,0,0 \\ 0, I, 0, 0 \end{bmatrix} \in \mathbb{R}^{r_{11} + r_{12} \times r}$

$P_W^{(2)} = \begin{bmatrix} 0, 0, I, 0 \\ 0, 0, 0, I \end{bmatrix}  \in \mathbb{R}^{r_{21} + r_{22} \times r}$

$P_V^{(1)} =
\begin{bmatrix}
I, 0, 0, 0 \\
0, 0, I, 0 
\end{bmatrix} \in \mathbb{R}^{r_{11} + r_{21} \times r}$

$P_V^{(2)} =\begin{bmatrix}
0, I, 0, 0\\
0, 0, 0, I \end{bmatrix} \in \mathbb{R}^{r_{12} + r_{22} \times r}$

are defined by construction. Note that expanded terms

$\bar{W}^{(1)}P_W^{(1)} = [W^{(11)}, W^{(12)}, 0 ,0]$

$\bar{W}^{(2)}P_W^{(2)} = [0,0, W^{(12)}, W^{(22)}]$

$\bar{V}^{(1)}P_V^{(1)} = [V^{(11)}, 0 , V^{(21)}, 0]$

$\bar{V}^{(2)}P_V^{(2)} = [ 0 , V^{(12)}, 0,  V^{(22)}]$

Intuitively, in a case with more blocks, the projections generalize this pattern of ``sliding" the block elements of $W$ and ``interleaving" the block elements of $V$ as defined in Alg. \ref{alg:blockwise_mixd_fact}.

All of the proofs in this work rely on this particular variant of SGD (a slight departure from the practical solver used in the experiments).

\subsection{Block-wise Rank Additivity}

Implicitly, we have assume a notion of \emph{rank additivity} over the block structure of $M$. Our notion of rank additivity used above is slightly less general than the one in \cite{baksalary1996note} but is sufficient for our purposes.

\begin{definition} \textbf{Rank Additive}
\emph{
Block matrix $M$ is \textbf{rank additive} if $r = \sum_{i=1}^{k_V}\sum_{j=1}^{k_W} r_{ij}$, where $r=\texttt{rank}(M)$ and $r_{ij}=\texttt{rank}(M^{(ij)})$.}
\end{definition}

Of course, rank additivity is a mild assumption when the ranks $r_{ij} \ll m$ and the number of blocks is asymptotically constant. In fact, it holds with high probability for standard random matrix models.

Let us show when the assumption of block-wise rank additivity holds for target matrix $M$. We begin by restating a relevant lemma \cite{tian2004rank,matsaglia1974equalities}. 

\begin{lemma} Let $A \in \mathbb{R}^{m\times n}, B \in \mathbb{R}^{m \times k}, C \in \mathbb{R}^{l \times n}$ and $D \in \mathbb{R}^{l \times k}$, while $\mathcal{R_M} = \texttt{range}(M)$ and $r_M = \texttt{rank}(M)$. Then,
$\texttt{rank}(\begin{bmatrix} A & B  \\
 C & D
\end{bmatrix}) = r_A + r_B + r_C + r_D$ 
iff 
$\mathcal{R}_A \cap \mathcal{R}_B = \mathcal{R}_C \cap \mathcal{R}_D = \mathcal{R}_{A^T} \cap \mathcal{R}_{C^T} = \mathcal{R}_{B^T} \cap \mathcal{R}_{D^T} = \{0\}$.
\end{lemma}

We generalize the above lemma in the proposition below. 

\begin{proposition}
\label{prop:block_add}
A block matrix $M$ is rank additive if 
$\mathcal{R}_{M^{(ij)}} \cap \mathcal{R}_{M^{(ij')}} =  \{0\}$ for all $1 \leq j \leq k_V$ and any $j \neq j'$

$\mathcal{R}_{M^{(ij)T}} \cap \mathcal{R}_{M^{(i'j)T}} = \{0\}$ for all $1 \leq i \leq k_W$ and any $i \neq i'$
\end{proposition}

\begin{proof}

Let $M^{(i)}$ be the $i$-th block-row of $M$. If for all $j \neq j'$, the $\texttt{range}(M^{(ij)}) \cap \texttt{range}(M^{(ij')}) = \{0\}$, then we directly obtain Eqn. 1: $$
\mathbf{dim}(\texttt{range}(M^{(i)})) = \sum_j \mathbf{dim}(\texttt{range}(M^{(ij)}))$$

Thus, each block-row is rank additive under the assumptions. With some minor additional effort, we can re-apply the above reasoning on the transpose: $M^T = [M^{(1)T},...,M^{(k_V)T}]$, treating the whole matrix as a block-row with $M^{(i)T}$ as the constituent blocks. 

Note that we have that $ \texttt{range}(M^{(i)T}) =  \bigoplus_{j=1}^{k_V}\texttt{range}(M^{(ij)T}) $
since $M^{(i)T}$ is a concatenation: $M^{(i)T} = [M^{(i0)T},...,M^{(ik_V)T}] $. 

Thus we have for $i \neq i'$:

$ \texttt{range}(M^{(i)T}) \cap \texttt{range}(M^{(i')T})  =$  $$\left( \bigoplus_{j=1}^{k_V}\texttt{range}(M^{(ij)T}) \right) \bigcap \left( \bigoplus_{j=1}^{k_V}\texttt{range}(M^{(i'j)T}) \right)  \subset$$
$$\bigoplus_{j=1}^{k_V}\left(\texttt{range}(M^{(ij)T}) \cap \texttt{range}(M^{(i'j)T})\right) = 
\bigoplus_j(\{0\}) = \{0\}$$

This implies

$$\{0\} \subset \texttt{range}(M^{(i)T}) \cap \texttt{range}(M^{(i')T}) $$
and thus  
$$\texttt{range}(M^{(i)T}) \cap \texttt{range}(M^{(i')T}) = \{0\}$$ 
from which we can conclude Eqn. 2: 

$$\mathbf{dim}(\texttt{range}(M^{T})) = \sum_{i}\mathbf{dim}(\texttt{range}(M^{(i)T}))$$

To conclude the proof, we have 

$\texttt{rank}(M) = \texttt{rank}(M^T) $

$=\mathbf{dim}(\texttt{range}(M^{T}))$ 

$=\sum_{i}\mathbf{dim}(\texttt{range}(M^{(i)T}))$ by Eqn. 2

$=\sum_{i}\texttt{rank}(M^{(i)T}) $

$=\sum_{i}\texttt{rank}(M^{(i)}) $

$=\sum_{i}\mathbf{dim}(\texttt{range}(M^{(i)})) $

$=\sum_{i}\sum_{j}\mathbf{dim}(\texttt{range}(M^{(ij)})$ by Eqn. 1

$=\sum_{i}\sum_{j}\texttt{rank}(M^{(ij)})$

\end{proof}

In other words, as long as the column and row spaces of these block matrices only intersect at the origin, rank additivity is attained. Of course, in a high-dimensional ambient space, randomly selected low-dimensional subspaces will not intersect beyond the origin from which it follows that rank additivity is in general a mild assumption that generally holds in practice.

\subsection{Data-Limited Regime}

We cover additional details about for the data-limited, as well as provide proofs for the associated theorems in this section.

\subsubsection{Extension to Tensor Completion}

As mentioned before, matrix completion is a common model for CF tasks. We assume the reader is familiar with this literature. We introduce the more general tensor completion problem as well. Tensor completion generalizes to contextual CF tasks and subsumes matrix completion as a special case. We review this here, following a setting similar to \cite{chen2013exact,zhang2019recovery}.   For tensor completion, the goal is recovering $T \in \mathbb{R}^{n_1 \times ... \times n_\kappa}$ where $T_{x_1,...,x_\kappa} \in \{0,1\}$ denotes if given context $(x_3,...,x_\kappa)$, user $x_1$ engages with item $x_2$. We assume $T$ has a low \emph{pairwise interaction rank} $r$, meaning $T$ can be factored into $\kappa$ matrices, $\mathcal{M}^{(i)} \in \mathbb{R}^{n_i \times r}$ for $i \in [\kappa]$ as follows: $$T_{x_1,...,x_\kappa} = \sum_{(i,j) \in [\kappa] \times [\kappa]} \langle \mathcal{M}^{(i)}_{x_i}, \mathcal{M}^{(j)}_{x_j} \rangle$$.

Under the assumed pairwise interaction rank $r$ for $T$, we can factor $T$ into $\kappa$ matrices, $\mathcal{M}^{(1)},...,\mathcal{M}^{(1)}$. we can adapt our model to the tensor case by exploiting the block structure and treating $M$ as a block pairwise interaction matrix rather than a preference matrix. Let $k_V = k_W = \kappa$ and let each block represent an interaction matrix: $M^{(ij)} = \mathcal{M}^{(i)} (\mathcal{M}^{(j)})^T$. Hence, $M$ is symmetric and with this simple construction, the factors of $T$ are represented as the blocks of $M$. The remaining distinction is that in the tensor case, the algorithm only observe sums of elements selected from the blocks of $M$ instead of observing the entries of $M$ directly. This minor distinction is addressed in both \cite{chen2013exact} and \cite{zhang2019recovery} and with appropriate care to details, the two observation structures are largely equivalent. For brevity, we discuss only the matrix case here, while keeping in mind this natural extension to the tensor case.  When thinking about the categorical features in CTR prediction, this construction is precisely the one we use to block by features, as described in Section 3.

\subsubsection{Assumptions}

Beyond the rank addivity assumption, we also implicitly assume a classical assumption on \emph{incoherence}.

The notion of \emph{incoherence} is central to matrix completion \cite{candes2010matrix,candes2010power,candes2009exact,candes2011tight, sun2016guaranteed, keshavan2010matrix}. Throughout this work, we implicitly assume that $M$ is $\mu$-incoherent. Note that asymptotic notation occults this. For many standard random models for $M$, $\mu$ scales like $\sqrt{r \log n}$ \cite{keshavan2010matrix}, but here, we simply take $\mu$-incoherence as an assumption on $M$. Note that all matrices are incoherent for some $\mu \in [1, \frac{\max \{n,m\} }{r}]$ \cite{sun2016guaranteed}.

\begin{definition}

 \textbf{Incoherence}. \emph{
Let $M = USV^T$ be the singular-value decomposition for a matrix rank-$r$ matrix $M \in \mathbb{R}^{n \times m}$. Matrix $M$ is $\mu$-incoherent if 
for all 
$1 \leq i \leq n, ||U_i||_2^2 \leq \frac{\mu r}{n}$ 
and for all 
$1 \leq j \leq m, ||V_j||_2^2 \leq \frac{\mu r}{n}$.}
\end{definition}

\paragraph{Low-sample Sub-matrix}

We denote $M_{\varepsilon}$ as the \emph{low sample sub-matrix} of blocks corresponding to minimum marginal  sampling rate $\varepsilon$. Concretely, $M_\varepsilon = [M^{(i_{\varepsilon}1)},...,M^{(i_{\varepsilon}k_V)}] \in \mathbb{R}^{n_{i_{\varepsilon}} \times m}$ if $\varepsilon_W \leq \varepsilon_V$ and $M_\varepsilon = [M^{(1j_{\varepsilon})},...,M^{(k_Wj_{\varepsilon})}] \in \mathbb{R}^{n \times m_{j_{\varepsilon}}}$ otherwise. For convenience, we also define $\Tilde{n}_\varepsilon = n_{i_\varepsilon}$ if $\varepsilon_W \leq \varepsilon_V$ and $\Tilde{n}_\varepsilon = m_{j_\varepsilon}$ otherwise. We refer to $\Tilde{n}_\varepsilon$ as the \textbf{size} of the low-sample sub-matrix (since the other dimension is inherited from the size of $M$ itself).

\subsubsection{Popularity-agnostic algorithms} (including UD matrix factorization) are those that can be seen as empirically matching at the observed indices at a given rank constraint, or any relaxation thereof, without taking advantage of popularity. MD factorization imposes additional popularity-based constraints. These additional constraints become essential to completion when the popularity is significantly skewed.

\begin{definition}\textbf{Popularity-Agnostic Algorithm.} \emph{
Let $f(\theta)$ be some arbitrary but fixed model with parameters $\theta$ that outputs an attempted reconstruction $\hat{M}$ of matrix $M$. For a given rank $r^*$ let $\mathcal{S}$ be any set of matrices such that for $\hat{S} = \{\hat{M} | rank(\hat{M}) = r^*\}$  we have $\hat{S} \subseteq \mathcal{S}$.
An algorithm $\mathcal{A}$ is popularity-agnositc if it outputs the solution to an optimization characterized by a Lagrangian of the form $\mathcal{L}(\mathbf{\theta,\lambda}) = ||M-\hat{M}||_\Omega^2 + \lambda\mathbf{1}[\hat{M} \not\in \mathcal{S}]$ where indicator function $\mathbf{1}[x] = 1$ if x=True and 0 otherwise.}
\end{definition}

\subsubsection{Popularity-Agnostic Completion Bounds}

It is standard to impose a low-rank structure in the context of matrix completion. We are interested in understanding how and when popularity-based structure can improve recovery. While, UD embeddings impose a low-rank structure, at a given rank, we can interpret our MD embeddings as imposing additional popularity-based constraints on the matrix reconstruction. While our MD embeddings maintain a particular rank, they do so with less parameters, thereby imposing an additional popularity-based restriction on the space of admissible solution matrices.

\paragraph{Non-asymptotic Upper Bound} We first give a simple lower bound on the sample complexity for popularity-agnostic algorithm. The bound is straightforward, based on the fact that without additional problem structure, in order to complete a matrix at rank $r$, you need at least $r$ observations, even on the least popular row or column. The global reconstruction efforts will always be thwarted by the least popular user or item. The bound below is non-asymptotic, holding for any problem size. The theorem below implies that popularity-agnostic algorithms pay steep recovery penalties depending on the least likely row/column to sample. If you want to exactly recover a matrix, you can only do as well as your most difficult row/column.

\begin{theorem}
Fix some $0< \delta < \frac{1}{2}$. Let $\varepsilon$ be the minimum marginal sampling rate and let $\Tilde{n}_\varepsilon$ be the size of the low-sample sub-matrix. Suppose number of samples $N \leq \frac{r}{\varepsilon}(1-\delta) $. Then, no popularity-agnostic algorithm can recover $M$ with probability greater than $\exp(-\frac{r\Tilde{n}_\varepsilon \delta^2}{3(1-\delta)})$.
\end{theorem}

\begin{proof}
Let $\psi$ be any one of the $\Tilde{n}_\varepsilon$ vectors in the low-probability sub-matrix $M_\varepsilon$ ($\psi$ is of length $n$ or $m$, depending on if $M_\varepsilon$ is a block-wise row or column).  Let $X_\psi$ be a random variable denoting the number of observations in corresponding to $\psi$ under block-wise Bernoulli sampling. Since $X_\psi$ is the sum of independent Bernoulli variables, we have $\mathbb{E}[X_\psi] = N\varepsilon$ by linearity of expectation. Furthermore, we require at least $r$ observations for each row and column in $M$ in order to achieve exact recovery at rank $r$. In order to see this directly, assume that an oracle completes all the embeddings except the row or column in question. Then, each observations defines immediately removes only one degree of freedom, since it defines the inner product with a known vector. It will be impossible to complete the final row or column with less than $r$ observations because popularity-agnostic algorithms provide no further constraints beyond a low-rank structure. Given that we need $r$ observations per vector, we can see that $\mathbb{E}[X_\psi] = N\varepsilon < r $. This implies if $N < r/\varepsilon$, we can use the Chernoff tail bound \cite{mitzenmacher2017probability} to bound $\mathbf{Pr}[X_\psi \geq r]$ from above. We take $1/2 < \delta < 1$ such that $N \leq \delta r /\varepsilon $. By application of the Chernoff bound, we have $\mathbf{Pr}[X_\psi \geq r] \leq \exp(-\frac{r}{3}\frac{\delta^2}{1-\delta})$. To complete the proof, notice that our argument extends to each of the $\Tilde{n}_\varepsilon$ vectors in $M_\varepsilon$ independently. Since all of these vectors require $r$ observations in order to complete of $M_\varepsilon$, we obtain the final probability by computing a product over the probability that each of $\Tilde{n}_\varepsilon$ vectors obtains at least $r$ observations.\end{proof}

\vspace{-4pt}
\paragraph{Asymptotic Upper Bound} We also provide a stronger asymptotic lower bound for exact completion, based on the results of \cite{candes2010power,keshavan2010matrix}. This lower bound assumes that the matrix size $n$ increases while keeping the sampling rate constant.  It includes an additional $O(\log n)$ factor, due to the well-known \emph{coupon collector} effect \cite{dawkins1991siobhan}.

Since $M$ is still a block matrix, we assume that asymptotically, each individual block becomes large, while $\Pi$ is held constant. More concretely, we assume each block scales at the same rate as the entire matrix: $n_{ij} = \Theta(n)$ for all $i,j$ and $m_{ij} = \Theta(n)$ for all $i,j$. In principle, we could also support an asymptotic number of blocks as well, as long as the number of blocks grows slowly enough compared to size of each block. Other numerical constants, such as the condition number and incoherence, are taken to be non-asymptotic. Note that we do not require the block additivity assumption for this to hold.

\begin{theorem}
\label{thm:neg2}
Let $M$ be a target matix following the block-wise Bernoulli sampling model. Let $\varepsilon$ be the minimum marginal sampling rate. Suppose $N = o(\frac{r}{\varepsilon}n \log n) $. Then any popularity-agnostic algorithm will have arbitrarily small probability of recovery, asymptotically.
\end{theorem}

\begin{proof}
Order $\Theta(nr \log n)$ observations are necessary for exact completion at a given probability in the asymptotic setting \cite{candes2010power,keshavan2010matrix}. This is because  $\Theta(nr \log n)$ observations are necessary to obtain $r$ observations per row and column. Each vector in the low-sample sub-matrix also requires $r$ observations. Since the number of samples in the low-sample sub-matrix concentrates around $N\varepsilon$, this number must be order $\Theta(rn \log n)$ in order to have a chance of reconstruction.\end{proof}

It is instructive to understand why UD embeddings fail to recover rank-$r$ matrix $M$ under popularity skew. For argument's sake, let  $d$ denote a potential uniform embedding dimension. Suppose we have $\Theta(rn\log n)$ samples and we set UD to $r$: $d \gets r$. When sampling is skewed, $M^{(2)}$ will be too sparsely covered to reveal $r$ degrees of freedom, since it only generates $\epsilon$ fraction of the observations. Thus, the $r$-dimensional embeddings would over-fit the under-constrained $M^{(2)}$ block as a result. Alternatively, if we set $d \gets \epsilon r$, as so to match the sample size over sub-matrix $M^{(2)}$, then our $\epsilon r$-dimensional embeddings would be unable to fit the larger training sample over the $M^{(1)}$ block. Namely, we would now have too many samples coming from a rank-$r$ matrix, resulting in an over-constrained problem that is infeasible with $\epsilon r$-dimensional embeddings. By using MD embeddings, we can avoid this problem by simultaneously fitting popular and rare blocks.

\subsubsection{Completion Guarantees for Mixed Dimension Embeddings}

In \cite{sun2016guaranteed} it was shown that various non-convex optimization algorithms, including SGD, could exactly complete the unknown matrix, under the Bernoulli sampling model. For convenience, the theorem is reproduced below. The details of the SGD implementation, such as step sizes and other parameters, can be found in \cite{sun2016guaranteed}.

\begin{theorem} (Sun and Luo, 2016)
\label{thm:sun2016}
Let $\alpha = \frac{n}{m} \geq 1$, $\kappa$ be the condition number of $M$ and $\mu$ be the incoherence of $M$. If (expected) number of samples $N \geq  C_0nr\kappa^2\alpha(\max\{\mu \log n, \sqrt{\alpha}\mu^2r^6\kappa^4\})$ then SGD completes $M$ with probability greater than $1 - \frac{2}{n^4}$.
\end{theorem}

We can use Thm \ref{thm:sun2016} and Alg. \ref{alg:blockwise_mixd_fact} to construct a guarantee for mixed-dimension block-wise factorization, as follows.

\begin{corollary}
\label{thm:complete_guarantee}
Let $M$ be a target matrix following the block-wise Bernoulli sampling model.
Let $C_0$ be a universal constant, $\hat{n}_{ij} = \min\{n_i,m_j\}$ and $N^*_{ij} = C_0\Pi_{ij}^{-1}\hat{n}_{ij}r_{ij}\kappa_{ij}^2\alpha_{ij}(\max\{\mu_{ij} \log \hat{n}_{ij}, \sqrt{\alpha_{ij}}\mu_{ij}^2r_{ij}^6\kappa_{ij}^4\})$ where $\kappa_{ij}$ and $\mu_{ij}$ is the condition number and the incoherence of the $ij$-th block of $M$ and $\alpha_{ij} = \frac{\max\{n_{ij},m_{ij}\}}{\min\{n_{ij},m_{ij}\}} \geq 1$. If $N \geq \max_{ij} N_{ij}^*$, then block-wise MD factorization completes rank additive block matrix $M$ with probability greater than $1 - \sum_{ij}\frac{2}{\hat{n}_{ij}^4}$.
\end{corollary}

\begin{proof}
Recall the construction used in Alg. \ref{alg:blockwise_mixd_fact}. First, we complete each block individually. We apply Thm \ref{thm:sun2016} to each block independently to guarantee its completion at rank $r_{ij}$ with probability at least $1-\frac{2}{n_{ij}^4}$. We then use the block-wise embeddings to construct MD embeddings $\bar{W}, \bar{V}$ as described  Alg. \ref{alg:blockwise_mixd_fact}. If $W^{(ij)}V^{(ij)T} = M^{(ij)}$, for all $i,j$, then $\bar{W}\bar{V}^T = M$. Thus, we need only a union bound on the failure probabilities from Thm \ref{thm:sun2016} to complete the proof.
\end{proof}

Note that Corollary \ref{thm:complete_guarantee} implies Thm \ref{thm:mix} as it is a non-asymptotic version. Namely, recall that $n_{ij} = \Theta(n)$ for all $i,j$. Furthermore, letting $C = C_0\max_{ij}{( \Pi_{ij}^{-1}r_{ij}\kappa_{ij}^2\alpha_{ij}\mu_{ij})}$ recovers Corollary \ref{thm:mix}.

\subsection{Memory-Limited Regime}
Now, we turn our attention to the allocation implications of non-uniformly sampled test sets. To abstract away training, we assume an oracle reveals the target matrix. Recall that our challenge is a small parameter budget --- our embeddings can only use $B$ parameters. The question is what dimensions $d_w$ and $d_v$ each embedding block should get to minimize our popularity-weighted reconstruction loss (under MSE)? Before proceeding, we pause to define some useful matrices from the block-wise MD factorization (Alg. \ref{alg:blockwise_mixd_fact}).

\begin{definition}
\emph{
In the context of \textbf{block-wise MD factorization} (Alg. \ref{alg:blockwise_mixd_fact}) we refer to the matrices $(W^{(ij)},V^{(ij)})$ as the  $ij$-th \textbf{block-wise embeddings}. We refer to matrix $\bar{W}^{(i)}$ as the $i$-th \textbf{row embedding block} and matrix $\bar{V}^{(j)}$ as the  $j$-th \textbf{column embedding block}.}
\end{definition}

Note that, generally speaking, embedding tables $(W,V)$, naturally inherit an \emph{embedding block} structure based on the block structure of $M$. For example, standard UD embeddings partition such that the top block-wise row $M^{(1)} = [M^{(11)},...,M^{(1k_W)}]$ only depends on $W^{(1)}$. Thus, embedding blocks exist independent of the usage of block-wise MD factorization. On the other hand, \emph{block-wise embeddings} $(W^{(ij)},V^{(ij)})$, are a distinct byproduct of block-wise MD factorization.

\subsubsection{Optimization over Embedding Dimensions}

We assume the block structure and block-wise probability matrix is given --- the variables over which the optimization takes place are 1) the dimensions of the embedding blocks, $(d_w,d_v)$ such that $W^{(l)} \in \mathbb{R}^{n_l \times (d_w)_l}$ and $V^{(l)} \in \mathbb{R}^{n_l \times (d_v)_l}$ and 2) the embedding blocks themselves $W^{(i)}$ for $1 \leq i \leq k_W$ and $V^{(j)}$ for $1 \leq j \leq k_V$. Note that when the embedding block dimensions are uniform, such that $(d_w)_i = (d_v)_j$ for all $i$ and $j$, this is equivalent to direct optimization over embedding matrices $W,V$ (i.e. matrix factorization). Recall that $L_\Pi$ is the popularity-weighted MSE. When $d_w$ and $d_v$ are treated as integers, this optimization is NP-Hard in general, since even integral feasibility under linear constraints is known to be NP-Hard \cite{karp1972reducibility}. Instead, we study a continuous relaxation that results in a convex program. The resultant convex program is far simpler, and yields a closed-form solution given the spectrum of the target matrix. We proceed to define another quantity of interest for our discussion, a \emph{spectral (singular value) decay}. In order to save space in the main text, we do not introduce the spectral decay rule $g$ but we imply it when referring to the spectrum directly. After the upcoming definition, we restate and prove Thm. \ref{thm:opt_sol} from the main text.

\begin{definition} \emph{
A \textbf{spectral decay} is mapping from $[0,r]$ to $\mathbb{R}^{+}$ that describes the singular value scree plot for a matrix. Let $\sigma_k$ be the $k$-th singular value of a matrix and $\sigma_k \ge \sigma_{k+1} $ for $k=1,...,r$. For any singular value spectrum we associate a spectral decay rule, a piece-wise step-function and its functional inverse, as $g(x) = \sigma_k$ for $k-1 \leq x < k$ and $g^{-1}(x) = k$ for $\sigma_{k} < x \leq \sigma_{k+1}$, respectively.}
\end{definition}

\begin{theorem}
The optimal block-wise embedding dimensions for the convex relaxation of the variable dimension embedding optimization under a parameter budget are given by
$d_{ij}^* = g_{ij}^{-1}\left( \sqrt{\lambda (n_i + m_j)(n_im_j)\Pi_{ij}^{-1}} \right)$
where $g_{ij}^{-1}$ is the functional inverse of the spectral decay of block $M^{(ij)}$.

\end{theorem}

\begin{proof}
The optimization is formulated as
$$\min_{d_w,d_v}\min_{W,V} L_{\Pi}(M,WV^T)$$ $$\text{s.t. } \sum_i n_i(d_w)_i + \sum_j m_j(d_v)_j \leq B$$

Under relaxation, we treat this a continuous optimization. Let $(k,l)$ be a test coordinate sampled according to $\Pi$. If rank additivity holds, we can equivalently write 

$$= \min_{d} \min_{W,V} \mathbb{E}_{(k,l) \in [n] \times [m]}  |M_{kl} - W_kV_l^T |^2$$ 

$$\text{ st } \sum_{ij} (n_i+m_j)d_{ij} \leq B$$

where $M_{kl}$ is the $kl$-th element of $M$, $(d_w)_i = \sum_{j}d_{ij}$ and $(d_v)_j = \sum_{i}d_{ij}$. The $d_w,d_v$ refer to the embedding block dimensions, whereas the $d_{ij}$ refer to the block-wise embedding dimensions (Definition C.11). We may ignore the parameters in the projections since they are not free parameters (and also contribute a negligible amount of parameters to the total count). Under Bernoulli sampling model, our popularity distribution yields.  As shorthand notation, let $\mathfrak{B} := \sum_{ij} (n_i+m_j)d_{ij}$.
\vspace{-2pt}

$$= \min_{d} \min_{W,V} \sum_{ij}  \frac{\Pi_{ij}}{n_im_j}    ||M^{(ij)} - W^{(ij)}V^{(ij)^T}||_F^2 \text{ st } \mathfrak{B} \leq B$$ 
\vspace{-2pt}

Since the constraints remain the same, we omit them for the time being. Letting $\sigma_k^{(ij)}$ be the singular values of block $M^{(ij)}$ and using the low-rank approximation theorem \cite{markovsky2008structured} we obtain
\vspace{-2pt}

$$= \min_{d} \sum_{ij} \frac{\Pi_{ij}}{n_im_j} \sum_{k=d_{ij}+1}^{r_{ij}} (\sigma_k^{(ij)})^2 \text{ st } \mathfrak{B} \leq B$$ 
\vspace{-2pt}

Letting  $g_{ij}$ be the spectral decay rule for each block and noticing that by construction $\sum_{k=0}^{r} \sigma_k = \int_0^{r} g(k) dk$ we obtain
\vspace{-2pt}

$$= \min_{d} \sum_{ij} \frac{\Pi_{ij}}{n_im_j}  \left( \int_0^{r_{ij}} g_{ij}^2(k) dk -  \int_0^{d_{ij}} g_{ij}^2(k) dk \right) \text{ st } \mathfrak{B} \leq B$$  
\vspace{-2pt}
$$= \min_{d} \sum_{ij} \frac{\Pi_{ij}}{n_im_j}  \left( ||M^{(ij)}||_F^2 -  \int_0^{d_{ij}} g_{ij}^2(k) dk \right) \text{ st } \mathfrak{B} \leq B$$  
\vspace{-2pt}

Observe that the objective is convex. To see this, note that each $g$ is decreasing since the spectral decay is decreasing. Thus, $g^2$ is also decreasing. The negative integral of a decreasing function is convex. Finally, since the objective is a sum of functions that are convex along one variable and constant along the rest, the entire optimization is convex (and well-posed under the linear constraint, which is guaranteed to be active). Thus we can solve with the optimization with first-order conditions \cite{luenberger1984linear}. The corresponding Lagrangian can be written as

$$ \mathcal{L} = \sum_{ij}  \frac{\Pi_{ij}}{n_im_j} \left( ||M^{(ij)}||_F^2 -  \int_0^{d_{ij}} g_{ij}^2(k) dk \right) $$ $$+ \lambda\left(-B + \sum_{ij}(n_i + m_j)d_{ij}\right)$$

Note that $M^{(ij)}$ does not depend on $d_{ij}$.  Also, note that we can use the fundamental theorem of calculus $\frac{\partial}{\partial x} \int_{0}^{x} g(t) dt = g(x) $ \cite{courant2012introduction}. Then, using Lagrange multipliers \cite{luenberger1984linear} we can write

$$\frac{\partial}{\partial d_{ij}} L_{\Pi}(M,WV^T) = - \frac{\Pi_{ij}}{n_i m_j}  g_{ij}^2(d_{ij}) + \lambda(n_i + m_j)$$
\vspace{-4pt}

Finally, using first order conditions  $\nabla_{d_{ij}} = [\frac{\partial}{\partial d_{ij}}] = 0$ we obtain: $g_{ij}^2(d_{ij}) = \lambda(n_i + m_j)(n_im_j)\Pi_{ij}^{-1}$. Solving for $d_{ij}$ by taking the functional inverse of $g_{ij}$ completes the proof. We conclude:  $$ d^*_{ij} =  g_{ij}^{-1}\left( \sqrt{\lambda (n_i + m_j)(n_im_j)\Pi_{ij}^{-1}} \right)$$\end{proof}

For specific spectral decay rules, we may give closed-form solutions, as done in the main text for power law decay. We can also analyze the performance gap between uniform and MD embeddings with respect to the optimization objective.

\begin{corollary}
\label{cor:perf_gap_emb}
The performance gap compared to UD embeddings is
$$\sum_{ij}\Pi_{ij} ( \mathbf{1}\{d_{ij}^{*}> \frac{B}{n+m}\} (\sum_{k=\frac{B}{n+m}}^{d^*_{ij}}(\sigma_k^{(ij)})^2)$$ $$ - \mathbf{1}\{d_{ij}^{*}< \frac{B}{n+m}\}(\sum_{k=d^*_{ij}}^{\frac{B}{n+m}}(\sigma_k^{(ij)})^2))
$$
\end{corollary}
\begin{proof}
Follows directly from plugging optimal $d^*$ into the objective. 
\end{proof}

We can explain the intuition for Corollary \ref{cor:perf_gap_emb} as follows. The first term counts the spectral mass gained back by allocating more parameters to frequent embeddings. $B /(n+m)$ is the embedding dimension under a uniform parameter allotment. When $d_{ij}^*$ is greater than this, we are increasing the dimension which enables that embedding to recover more of the spectrum. This occurs when $\Pi_{ij}$ is large. On the other hand, the trade-off is that lower-dimension embeddings recover less of the spectrum when $\Pi_{ij}$ is small, which is the penalty incurred by the second term.

\begin{corollary}
When $M$ exhibits a block-wise power spectral decay, this becomes:
\vspace{-5pt}
$$d_{ij}^* =  \lambda\zeta_{ij}\Pi_{ij}^{\frac{1}{2\beta}} $$

where $\zeta_{ij} = \left(\frac{(n_i+m_j)(n_im_j)}{\mu}\right)^{\frac{-1}{2\beta}}$ and $\lambda = \left(\frac{B}{\sum_{ij}(n_i+m_j)\zeta_{ij}}\right)^{-2\beta} $
\end{corollary}

\begin{proof}
Follows directly by substituting power spectral decay rule for $\sigma(\cdot)$. 
\end{proof}

\subsubsection{Approximation Gap for Convex Relaxation}

Note that we can bound the approximation gap of the proposed relaxation by simply rounding down each $d_{ij}$ to the nearest integer, which ensures the feasibility of the assignment. The absolute approximation error is then less than $\sum_{ij} \Pi_{ij} \cdot g^2(d_{ij})$. For most applications, this quantity is small for a good MD assignment, since either the probability term, $\Pi_{ij}$ is small, or the the spectrum at $d_{ij}$ is small. For example, in typical use cases, the embedding dimensions may be on the order of $10-100$ -- rounding down to the nearest integer would thus represent a loss of $10-1\%$ of the spectral mass.

\end{document}